\documentclass[twoside]{article}

\usepackage[accepted]{aistats2024}
\usepackage[round]{natbib}

\usepackage[utf8]{inputenc} %
\usepackage[T1]{fontenc}    %
\usepackage{hyperref}       %
\usepackage{url}            %
\usepackage{booktabs}       %
\usepackage{amsfonts,amsmath,amssymb}       %
\usepackage{nicefrac}       %
\usepackage{microtype}      %
\usepackage{xcolor}         %
\usepackage{wrapfig}
\usepackage{algorithm}

\usepackage{graphicx}
\usepackage{enumitem}
\usepackage{subcaption}
\usepackage{hyperref}
\usepackage[font={small}]{caption}
\usepackage{titlesec}

\usepackage{mathtools}
\usepackage{multirow}

\usepackage{amsthm}

\usepackage[createShortEnv,conf={no link to proof, text link},commandRef=Cref]{style/proof-at-the-end}

\usepackage[commentColor=black,beginLComment=/*~, endLComment=~*/]{algpseudocodex}

\theoremstyle{plain}
\newtheorem{theorem}{Theorem}[section]
\newtheorem{proposition}[theorem]{Proposition}
\newtheorem{lemma}[theorem]{Lemma}
\newtheorem{corollary}[theorem]{Corollary}
\theoremstyle{definition}
\newtheorem{definition}[theorem]{Definition}

\theoremstyle{remark}

\newcommand{\cut}[1]{}

\usepackage[capitalize,noabbrev]{cleveref}
\makeatletter
\crefname{section}{\S\@gobble}{\S\S\@gobble}
\crefname{subsection}{\S\@gobble}{\S\S\@gobble}
\makeatother

\makeatletter
\renewcommand*{\sectionautorefname}{\S\@gobble}
\renewcommand*{\subsectionautorefname}{\S\@gobble}
\renewcommand*{\subsubsectionautorefname}{\S\@gobble}
\makeatother

\usepackage[font={small}]{caption}

\makeatletter
\def\paragraph{\textbf}
\def\subparagraph{\textit}

\makeatother

\newcommand{\nothat}[1]{#1^\circ}

\usepackage{amsmath,amsfonts,bm}

\def\eqref#1{eq.~\ref{#1}}

\def\1{\bm{1}}

\def\rva{{\mathbf{a}}}
\def\rvb{{\mathbf{b}}}

\def\rvw{{\mathbf{w}}}
\def\rvx{{\mathbf{x}}}
\def\rvy{{\mathbf{y}}}

\DeclareMathAlphabet{\mathsfit}{\encodingdefault}{\sfdefault}{m}{sl}
\SetMathAlphabet{\mathsfit}{bold}{\encodingdefault}{\sfdefault}{bx}{n}

\def\gL{{\mathcal{L}}}

\def\gN{{\mathcal{N}}}
\def\gO{{\mathcal{O}}}

\def\gU{{\mathcal{U}}}

\def\sP{{\mathbb{P}}}
\def\sQ{{\mathbb{Q}}}

\def\sW{{\mathbb{W}}}

\newcommand{\E}{\mathbb{E}}

\newcommand{\R}{\mathbb{R}}

\usepackage{xcolor}

\newcommand{\ie}{\textit{i.e.}}
\newcommand{\eg}{\textit{e.g.}}

\newcommand{\Pa}{\mathcal{P}_m(\alpha_n)}
\newcommand{\Pb}{\mathcal{P}_m(\beta_n)}
\newcommand{\an}{\alpha _n}
\newcommand{\bn}{\beta _n}
\usepackage{pifont}
\definecolor{myred}{RGB}{215,48,39}
\definecolor{mygreen}{RGB}{26,152,80}
\newcommand{\cmark}{\textcolor{mygreen}{\ding{51}}}
\newcommand{\xmark}{\textcolor{myred}{\ding{55}}}
\newcommand{\halfmark}{\textcolor{gray}{\checkmark\kern-1.1ex\raisebox{.7ex}{\rotatebox[origin=c]{125}{--}}}}

\newcommand{\ds}[1]{{\sffamily {#1}}}

\runningauthor{A.\ Tong, N.\ Malkin, K.\ Fatras, L.\ Atanackovic, Y.\ Zhang, G.\ Huguet, G.\ Wolf, Y.\ Bengio}

\begin{document}

\twocolumn[

\aistatstitle{Simulation-Free Schr\"odinger Bridges via Score and Flow Matching}

\aistatsauthor{Alexander Tong$^\dagger$\\
Mila -- Qu\'ebec AI Institute\\Universit\'e de Montr\'eal
\And
Nikolay Malkin$^\dagger$\\
Mila -- Qu\'ebec AI Institute\\Universit\'e de Montr\'eal
\And
Kilian Fatras$^\dagger$\\
Mila -- Qu\'ebec AI Institute\\McGill University
\AND
Lazar Atanackovic\\
University of Toronto\\Vector Institute
\And
Yanlei Zhang\\
Mila -- Qu\'ebec AI Institute\\Universit\'e de Montr\'eal
\And
Guillaume Huguet\\
Mila -- Qu\'ebec AI Institute\\Universit\'e de Montr\'eal
\AND
Guy Wolf\\
Mila -- Qu\'ebec AI Institute\\Universit\'e de Montr\'eal\\Canada CIFAR AI Chair
\And
Yoshua Bengio\\
Mila -- Qu\'ebec AI Institute\\Universit\'e de Montr\'eal\\CIFAR Senior Fellow}
\vspace{0.5cm}
]

\begin{abstract}
We present \emph{simulation-free score and flow matching} ([SF]$^2$M), a simulation-free objective for inferring stochastic dynamics given unpaired samples drawn from arbitrary source and target distributions. Our method generalizes both the score-matching loss used in the training of diffusion models and the recently proposed flow matching loss used in the training of continuous normalizing flows. [SF]$^2$M interprets continuous-time stochastic generative modeling as a Schr\"odinger bridge problem. It relies on static entropy-regularized optimal transport, or a minibatch approximation, to efficiently learn the SB without simulating the learned stochastic process. We find that [SF]$^2$M is more efficient and gives more accurate solutions to the SB problem than simulation-based methods from prior work. Finally, we apply [SF]$^2$M to the problem of learning cell dynamics from snapshot data. Notably, [SF]$^2$M is the first method to accurately model cell dynamics in high dimensions and can recover known gene regulatory networks from simulated data. Our code is available in the \texttt{TorchCFM} package at \url{https://github.com/atong01/conditional-flow-matching}.
\end{abstract}

\section{INTRODUCTION}

Score-based generative models (SBGMs), including diffusion models, are a powerful class of generative models that can represent complex distributions over high-dimensional spaces \citep{sohl-dickstein_deep_2015,song_generative_2019,ho_denoising_2020,nichol_improved_2021,dhariwal_diffusion_2021}. SBGMs typically generate samples by simulating the evolution of a source density -- nearly always a Gaussian -- according to a stochastic differential equation (SDE) \citep{song_score-based_2021}. Despite their empirical success, SBGMs are restricted by their assumption of a Gaussian source, which is essential for optimization with the simulation-free denoising objective. This assumption is often violated in the temporal evolution of physical or biological systems, such as in the case of single-cell gene expression data, which prevents the use of SBGMs for learning the underlying dynamics.

\begin{figure*}[t]
\includegraphics[width=0.6\textwidth]{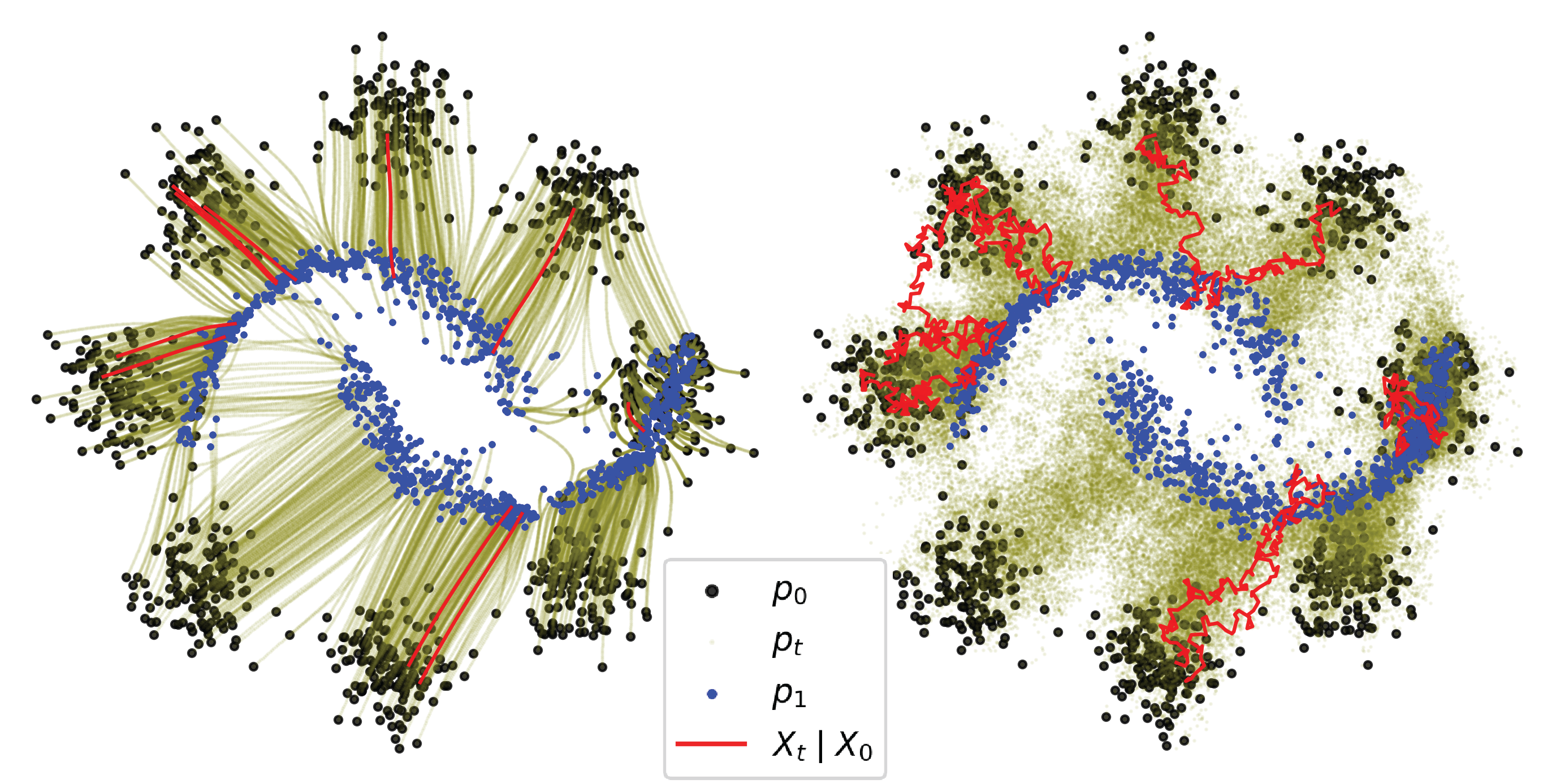}
\includegraphics[width=0.4\textwidth,trim=0 10 0 10]{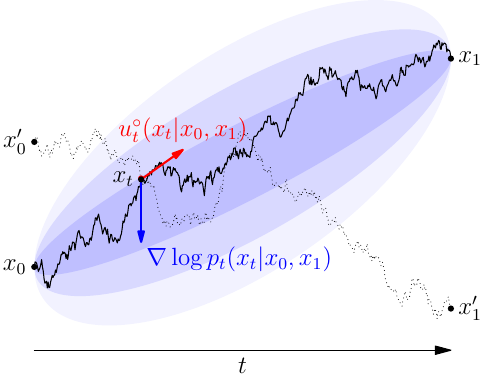}
\caption{\small{\textbf{Left:} ODE and SDE paths from \ds{8gaussians} to \ds{moons}, sampled from a model trained using [SF]$^2$M. [SF]$^2$M makes it possible to vary the diffusion schedule at inference time and thus interpolate between ODEs and SDEs that have the same marginal densities. \textbf{Right:} Illustration of the stochastic regression objective in [SF]$^2$M. Given a source point $x_0$ and target point $x_1$ sampled from an entropic OT plan between marginals, an intermediate point $x_t$ is sampled from the Brownian bridge (marginal in light blue) in a simulation-free way. Neural networks are regressed to the ODE drift $\nothat u_t(x_t|x_0,x_1)$ and to the conditional score $\nabla\log p_t(x_t|x_0,x_1)$. The regression objective is stochastic, as the same point $x_t$ may appear on different conditional paths, \eg, the dotted path from $x_0'$ to $x_1'$. The stochastic regression recovers dynamics that transform the marginal at time 0 to that at time 1.}}
\label{fig:figure_one}
\end{figure*}

An approach of choice in such problems has been to use flow-based generative models, synonymous with continuous normalizing flows (CNFs) \citep{chen_neural_2018,grathwohl_ffjord:_2019,finlay_how_2020}. Flow-based models assume a \emph{deterministic} continuous-time generative process and fit an ordinary differential equation (ODE) that transforms the source density to the target density. Flow-based models were previously limited by inefficient simulation-based training objectives that require an expensive integration of the ODE at training time. However, recent work has introduced simulation-free training objectives that make CNFs competitive with SBGMs when a Gaussian source is assumed \citep{lipman_flow_2022,liu_flow_2023,pooladian_2023_multisample} and extended these objectives to the case of arbitrary source distributions \citep{liu_rectified_2022,albergo_building_2023,tong_conditional_2023}. However, these objectives do not yet apply to learning stochastic dynamics, which can be beneficial both for generative modeling and for recovering the dynamics of systems.

The Schr\"odinger bridge (SB) problem -- the canonical probabilistic formulation of stochastically mapping between two arbitrary distributions -- considers the most likely evolution between a source and target probability distributions under a given reference process \citep{schrodinger,leonard_schrodinger}. The SB problem has been applied in a wide variety of problems, including generative modeling~\citep{de_bortoli_diffusion_2021,vargas_solving_2021,chen_likelihood_2022,wang_deep_2021,song_generative_2019}, modeling natural stochastic dynamical systems~\citep{schiebinger_optimal-transport_2019,holdijk_path_2022,koshizuka_neural_2022}, and mean field games~\citep{liu_deep_2022}. Except for a small number of special cases (\eg,\ Gaussian~\citep{mallasto_entropy_2021,bunne_recovering_2022}), the SB problem typically does not have a closed-form solution, but can be approximated with iterative algorithms that require simulating the learned stochastic process \citep{de_bortoli_diffusion_2021,chen_likelihood_2022,bunne_recovering_2022}. While theoretically sound, these methods present numerical and practical issues that limit their scalability to high dimensions~\citep{shi_diffusion_2023}.

\textbf{This paper introduces a simulation-free objective for the Schr\"odinger bridge problem} called \emph{simulation-free score and flow matching} ([SF]$^2$M). [SF]$^2$M simultaneously generalizes (1) the simulation-free objectives for CNFs \citep{tong_conditional_2023,liu_flow_2023} to the case of stochastic dynamics and (2) the denoising training objective for diffusion models to the case of arbitrary source distributions (\autoref{fig:figure_one}). Our algorithm uses a connection between the SB problem and entropic optimal transport (OT) to express the Schr\"odinger bridge as a mixture of Brownian bridges \citep{de_bortoli_diffusion_2021,leonard_schrodinger}. In contrast to dynamic SB algorithms that require simulating an SDE at every iteration, [SF]$^2$M takes advantage of \emph{static} entropic OT maps between source and target distributions, which are efficiently computed by the Sinkhorn algorithm \citep{sinkhorn_relationship_1964}. 

We demonstrate the effectiveness of [SF]$^2$M on both synthetic and real-world datasets. On synthetic data, we show that [SF]$^2$M performs better than related prior work and finds a better approximation to the true Schr\"odinger bridge. As an application to real data, we consider modeling sequences of cross-sectional measurements (\ie, unpaired time series observations) by a sequence of Schr\"odinger bridges. While there are many prior methods on modeling cells with Schr\"odinger bridges in the static setting~\citep{schiebinger_optimal-transport_2019,huguet_geodesic_2022,Lavenant_towards_2023,nolan_Brassinosteroid_2023} or low-dimensional dynamic setting~\citep{bunne_proximal_2022,bunne_recovering_2022,koshizuka_neural_2022}, [SF]$^2$M is the first method able to scale to thousands of gene dimensions, as its training is completely simulation-free. We also introduce a static manifold geodesic map which improves cell interpolations in the dynamic setting, demonstrating one of the first practical applications of Schr\"odinger bridge approximations with non-Euclidean costs. Finally, we show that unlike with static optimal transport, we are able to directly model and recover the gene-gene interaction network driving the cell dynamics. 

We summarize our \textbf{main contributions} below:

\begin{itemize}[noitemsep,topsep=0pt,parsep=0pt,partopsep=0pt,label={\large\textbullet},leftmargin=*]
    \item \looseness=-1 We present [SF]$^2$M, the first simulation-free objective for the Schr\"odinger bridge problem, and prove its correctness.
    \item We study effective empirical and minibatch approximations to the entropic OT plan used in [SF]$^2$M.
    \item We validate our proposed method on synthetic distributions and in several single-cell dynamics problems. 
\end{itemize}

\section{PRELIMINARIES}\label{sec:prelim}

We consider a pair of compactly supported distributions over $\R^d$ with (unknown) densities $q(x_0)$ and $q(x_1)$ (also denoted $q_0,q_1$). We assume access to a finite dataset of samples from $q_0$ and $q_1$. The problem of continuous-time stochastic generative modeling, or SDE inference, consists in finding a stochastic mapping $f$ that transforms $q_0$ to $q_1$. Samples from $q_1$ can then be generated by drawing a sample from $q_0$ and applying $f$ to obtain a sample from $q_1$.

\subsection{Schr\"odinger bridges via entropic OT}

The Schr\"odinger bridge problem asks to find most likely evolution between two probability measures $q_0$ and $q_1$ with respect to a reference stochastic process $\mathbb{Q}$. Formally, the Schr\"odinger bridge is the solution of:
\begin{equation}\label{eq:sb}
    \mathbb{P}^\star = \underset{\mathbb{P}: p_0 = q_0, p_1 = q_1}{\operatorname{arg\ min}} \operatorname{KL}(\mathbb{P}\,\|\,\mathbb{Q}),
\end{equation}
where $\mathbb{P}$ is a stochastic process (distribution over continuous paths $[0,1]\to\R^d$) with law $p$ (with marginals denoted $p_t$). 

\paragraph{SDEs and diffusion processes.} The stochastic processes we consider can be represented as It\^o SDEs of the form $dx = u_t(x)\,dt + g(t)\,dw_t$, where $u_t$ is a smooth vector field and $dw_t$ a Brownian motion. A density $p(x_0)$ evolved according to a SDE induces marginal distributions $p_t(x_t)$, viewed as a function $p:[0,1]\times\R^d\to\R_+$. They are characterized by the initial conditions $p_0$ and the \emph{Fokker-Planck equation} 
    $\partial_t p_t=-\nabla\cdot(p_tu_t)+\frac{g^2(t)}{2}\Delta p_t,$
where $\Delta p_t=\nabla\cdot(\nabla p_t)$ is the Laplacian.

In this work, we consider $\mathbb{Q} = \sigma \mathbb{W}$, where $\mathbb{W}$ is the standard Brownian motion defined by the SDE $dx=dw_t$, in which case the solution to (\ref{eq:sb}) is known as the \textit{diffusion Schr\"odinger bridge}~\citep{de_bortoli_diffusion_2021,bunne_recovering_2022}. We refer to \cite{leonard_path_measures, leonard_schrodinger} for a full discusssion on Schr\"odinger bridges. 

\paragraph{Entropically-regularized optimal transport.} The entropic OT problem is defined as follows:
\begin{align}\label{eq:eot}
    &\operatorname{\pi}_\varepsilon^*(q_0, q_1) = \\
    &\underset{\pi \in U(q_0, q_1)}{\arg\,\min} \int d(x_0,x_1)^2\, d\pi(x_0,x_1) + \varepsilon \operatorname{KL}(\pi \| q_0 \otimes q_1), \nonumber
\end{align}
where $U(q_0, q_1)$ is the set of admissible transport plans (joint distributions over $x_0$ and $x_1$ whose marginals are equal to $q_0$ and $q_1$), $d(\cdot,\cdot)$ is the ground cost, $\varepsilon$ is the regularization parameter, and $q_0\otimes q_1$ is the joint distribution over $x_0,x_1$ in which $x_0$ and $x_1$ are independent.
When $\varepsilon \rightarrow 0$, we recover \emph{exact optimal transport}. 
We now recall a cornerstone theorem that connects the SB problem to the entropic OT plan:
\begin{proposition}[\cite{Follmer1988}]\label{prop:sb_is_mixture_of_bb}
    Let the reference process be a Brownian motion (\emph{\ie,} $\mathbb{Q} = \sigma \mathbb{W}$). Then the Schr\"odinger bridge problem admits a unique solution $\sP^*$ having the form of a mixture of Brownian bridges weighted by an entropic OT plan:
    \begin{equation}\label{eq:bb_mixture}
        \!\!\!\!\!\!\!\sP^*((x_t)_{t\in[0,1]})\!=\!\int \sQ((x_t)_{t}\mid   x_0,x_1)\,d\pi^\star_{2 \sigma^2}(x_0, x_1)
    \end{equation}
    where $\sQ((x_t)_{t\in(0,1)}\mid x_0,x_1)$ is the Brownian bridge between $x_0$ and $x_1$ with diffusion rate $\sigma$.
\end{proposition}
Motivated by this result, the algorithm we propose stochastically regresses the parameters defining an unconditional SDE to those defining Brownian bridges.

\subsection{Neural SDEs and probability flows}

In this section, we consider an SDE $dx = u_t(x)\,dt + g(t)\,dw_t$. We review some important properties and discuss the approximation of $u_t$ by a neural network.

\paragraph{Score and flow parametrization.}
In the degenerate case $g(t)\equiv0$, an SDE becomes an ODE and the Fokker-Planck equation recovers the \emph{continuity equation} $\frac{\partial p}{\partial t}=-\nabla\cdot(p_tu_t)$. From the Fokker-Plank and continuity equations, it can easily be derived that the ODE
\begin{equation}\label{eq:pfode}
    dx=\underbrace{\left[u_t(x)-\frac{g(t)^2}{2}\nabla\log p_t(x)\right]}_{\nothat u_t(x)}\,dt,\vspace*{-2mm}
\end{equation}
together with a distribution over initial conditions $p(x_0)$, induces the same marginal distributions $p_t(\cdot)$ as the SDE; therefore, (\ref{eq:pfode}) is called the \emph{probability flow ODE} of the stochastic process. Conversely, if the probability flow ODE's drift $\nothat u_t(x)$, the diffusion schedule $g(\cdot)$, and the \emph{score function} $\nabla\log p_t(x)$ are known, then the SDE's drift term can be recovered via
\begin{equation}\label{eq:sde_from_sf}
    u_t(x)=\nothat u_t(x)+\frac{g(t)^2}{2}\nabla\log p_t(x).
\end{equation}
Therefore, \textbf{specifying an SDE is tantamount to specifying the probability flow ODE and its score function}. By reversing the sign of $\nothat u_t$ in the ODE (\ref{eq:pfode}) and converting it to an SDE using (\ref{eq:sde_from_sf}), we also get the time reversal formula from \citet{anderson_1982}:
\begin{align}\label{eq:reverse_sde}
    dx&=\left[-\nothat u_t(x)+\frac{g(t)^2}{2}\nabla\log p_t(x)\right]dt+g(t)\,dw_t \nonumber
    \\&=\left[-u_t(x)+g(t)^2\nabla\log p_t(x)\right]dt+g(t)\,dw_t,
\end{align}
which induces the same distribution on $x_{1-t}$ as the original SDE does on $x_t$.

\paragraph{Approximating SDEs with neural networks.} If the marginal $p_t(x)$ can be tractably sampled and one knows the probability flow ODE's drift  $\nothat u_t(x)$ and score $\nabla\log p_t(x)$, both can be approximated by neural networks. Specifically, time-varying vector fields $v_\theta(\cdot,\cdot):[0,1]\times\R^d\to\R^d$ and $s_\theta(\cdot,\cdot):[0,1]\times\R^d\to\R_d$ can be trained with the \emph{(unconditional) score and flow matching} objective\vspace*{-1mm}
\begin{align}\label{eq:uncond_sf2m}
    &\gL_{\rm U[SF]^2M}(\theta) = \\
    &\E\big[\underbrace{\|v_\theta(t,x)-\nothat u_t(x)\|^2}_{\text{flow matching loss}}+\lambda(t)^2\underbrace{\|s_\theta(t,x)-\nabla\log p_t(x)\|^2}_{\text{score matching loss}}\big],\nonumber\vspace*{-1mm}
\end{align}
where the expectation is over $t\sim\gU(0,1),x\sim p_t(x)$ and $\lambda(\cdot)$ is some choice of positive weights. (In practice, it can be more stable to approximate $g(t)^2\nabla\log p_t(x)$ rather than  $\nabla\log p_t(x)$, a simple parametrization change that does not change the learning problem or the objective.) Once trained, $v_\theta$ and $s_\theta$ can be used to simulate the SDE from source samples $x_0$. This procedure is described in \autoref{alg:sf2m-inference}.

Remarkably, with a separate parametrization of the probability flow ODE and the score, we can simulate the SDE at inference time with an arbitrary diffusion rate $g(\cdot)$ that need not match the one used at training time. If the global optimum of (\ref{eq:uncond_sf2m}) is attained (\autoref{fig:figure_one}), we are ensured to obtain samples from the same marginals for any arbitrary diffusion rate $g(\cdot)$. For example, we can simulate the probability flow ODE by setting $g(t)\equiv0$. Similarly, the backward SDE can be simulated starting at samples $x_1$ using the time reversal formula (\ref{eq:reverse_sde}).

\paragraph{ODEs and SDEs for Brownian bridges.} For processes whose marginals $p_t(x)$ are Gaussian, Theorem 3 of \citet{lipman_flow_2022} or Theorem 2.1 of \citet{tong_conditional_2023} yield expressions for the flow and score. The main case of interest is the Brownian bridge from $x_0$ to $x_1$ with constant diffusion rate $g(t)=\sigma$. The marginals are given by $p_t(x)=\gN(x;tx_1+(1-t)x_0,\sigma^2t(1-t))$, and the ODE and score are computed using the aforementioned result:
\begin{align}\label{eq:brownian_bridge_flow}
    &\nothat u_t(x)=\frac{1-2t}{t(1-t)}(x-(tx_1+(1-t)x_0))+(x_1-x_0), \nonumber\\
    &\quad\nabla\log p_t(x)=\frac{tx_1+(1-t)x_0-x}{\sigma^2 t(1-t)}.
\end{align}
The Schr\"odinger bridge approximation algorithm we will propose leverages fast solutions to the entropy-regularized OT problem (\ref{eq:eot}) and the closed-form $\nothat{u}_t$ and $\nabla \log p_t$ of Brownian bridges (\ref{eq:brownian_bridge_flow}).

\section[Simulation-free SDE training]{SIMULATION-FREE SDE TRAINING}\label{sec:method}

We next describe our simulation-free method to learn SDEs through score and flow matching, summarized in \autoref{alg:sf2m-train}. We present the general case in \autoref{sec:alg:general}, then consider the Schr\"odinger bridge case in \autoref{sec:solving_sb}.

\subsection{Matching the conditional flow and score}\label{sec:alg:general}

\citet{tong_conditional_2023} described a simulation-free stochastic regression objective, conditional flow matching (CFM), that fits an ODE generating marginal distributions given by mixtures of simpler probability paths. We generalize CFM to matching stochastic dynamics.

Suppose the stochastic process $\sP((x_t)_{t\in[0,1]})$, with marginals $p_t(x)$, is a mixture over a latent variable $z$ with density $q(z)$, \ie,
\begin{equation}\label{eq:pt_mixture}
    \sP((x_t)_{t\in[0,1]})=\int \sP((x_t)|z)q(z)\,dz.
\end{equation}
Suppose that $\sP((x_t)|z)$ is defined by the SDE $dx=u_t(x|z)\,dt+g(t)\,dw_t$ with initial conditions $p_0(x|z)$, and let $\nothat u_t(x|z)$ be drift of the corresponding probability flow ODE given by (\ref{eq:pfode}). One then has expressions for the probability flow ODE and score that generate the marginals of the process $\sP$ given initial conditions $p_0(x)=\int_{z}p_0(x|z)q(z)\,dz$:
\begin{align}\label{eq:sde_for_mixture}
    \nothat u_t(x)&=\E_{q(z)}\frac{\nothat u_t(x|z)p_t(x|z)}{p_t(x)},\nonumber \\
    \quad\nabla\log p_t(x)&=\E_{q(z)}\left[\frac{p_t(x|z)}{p_t(x)}\nabla\log p_t(x|z)\right].
\end{align}
To be precise, we generalize Theorem 3.1 from \citet{tong_conditional_2023} to stochastic settings:
\begin{theoremE}[][end,restate]\label{thm:score_and_flow_of_mixture}
    Under mild regularity conditions, the ODE $dx=\nothat u_t(x)\,dt$ generates the marginals $p_t$ of $\sP$ from initial conditions $p_0$, and the score is given by (\ref{eq:sde_for_mixture}). The SDE $dx=[u_t^\circ(x) + \frac12 g(t)^2 \nabla\log p_t(x)]\,dx+g(t)\,dt$ generates the Markovization of $\sP$.
\end{theoremE}
\begin{proofE}
We assume $u_t(x|z)$, $\nabla\log p_t(x|z)$, and their first derivatives are continuous in $x$, uniformly in $z$. The first statement is Theorem 3.1 from \citet{tong_conditional_2023}, but we reproduce the key derivation here for completeness:
\begin{align*}
\frac{d}{dt} p_t(x)
&= \frac{d}{dt} \int p_t(x | z)q(z) dz \\
&= \int \frac{d}{dt} \left (p_t(x | z)q(z) \right ) dz \\
&= -\int \nabla\cdot \left (\nothat u_t(x | z) p_t(x | z)q(z) \right ) dz \\
&= -\nabla\cdot \left ( \int \nothat u_t(x | z) p_t(x | z)q(z)  dz \right )\\
&= -\nabla\cdot \left (\nothat u_t(x) p_t(x) \right ),
\end{align*}
showing that $p_t$ and $\nothat u_t$ satisfy the continuity equation $\frac{d}{dt} p_t(x) +\nabla\cdot \left (\nothat u_t(x) p_t(x) \right ) = 0$, which implies that $\nothat u_t$ generates $p_t$ from initial conditions $p_0$.

To show the score is given by the expression in (\ref{eq:sde_for_mixture}), using that $p_t(x)=E_{q(z)}p_t(x|z)$, we have
\begin{align*}
\nabla\log p_t(x)
&=\frac{\nabla\E_{q(z)}[p_t(x|z)]}{p_t(x)}\\
&=\frac{\E_{q(z)}[\nabla p_t(x|z)]}{p_t(x)}\\
&=\E_{q(z)}\left[\frac{p_t(x|z)\nabla\log p_t(x|z)}{p_t(x)}\right],
\end{align*}
as desired.

The Markovization of a mixture of SDEs $\sP$ with a common diffusion rate $g(t)$ is an SDE $dx=u_t(x)\,dt+g(t)\,dw$, where
\[u_t(x_t)=\lim_{\Delta t\to0}\E_{x_{t+\Delta t}\sim\sP(x_{t+\Delta t}\mid x_t)}\left[\frac{x_{t+\Delta t}-x_t}{\Delta t}\right].\]
Using the definition of $\sP$ as a mixture in (\ref{eq:pt_mixture}), we decompose this expectation over the posterior $p(z\mid x_t)=\frac{p_t(x_t\mid z)q(z)}{p_t(x_t)}$:
\begin{align*}
    \lim_{\Delta t\to0}\E_{x_{t+\Delta t}\sim\sP(x_{t+\Delta t}\mid x_t)}\left[\frac{x_{t+\Delta t}-x_t}{\Delta t}\right]
    &=
    \lim_{\Delta t\to0}\E_{z\sim p(z\mid x_t)}\E_{x_{t+\Delta t}\sim\sP(x_{t+\Delta t}\mid x_t,z)}\left[\frac{x_{t+\Delta t}-x_t}{\Delta t}\right]\\
    &=
    \E_{z\sim p(z\mid x_t)}\lim_{\Delta t\to0}\E_{x_{t+\Delta t}\sim\sP(x_{t+\Delta t}\mid x_t,z)}\left[\frac{x_{t+\Delta t}-x_t}{\Delta t}\right]\\
    &=
    \E_{z\sim p(z\mid x_t)}\left[u_t(x_t\mid z)\right].
\end{align*}
The score and probability flow defined in (\ref{eq:sde_for_mixture}) can equivalently be expressed as
\[
    u_t^\circ(x_t)=\E_{z\sim p(z\mid x_t)}u_t^\circ(x_t\mid z),\quad\nabla\log p_t(x_t)=\E_{z\sim p(z\mid x_t)}\nabla\log p_t(x_t\mid z),
\]
which, together with $u_t(x\mid z)=u_t^\circ(x\mid z)+\frac12g(t)^2\nabla\log p_t(x\mid z)$, implies that
\[u_t(x)=\E_{z\sim p(z\mid x)}\left[u_t(x\mid z)\right]=u_t^\circ(x)+\frac12g(t)^2\nabla\log p_t(x),\]
as desired.
\end{proofE}
We emphasize that, in general, the SDE in \autoref{thm:score_and_flow_of_mixture} does not recover $\sP$, but only its Markovization (\ie, the process with the same infinitesimal transition kernel). A process $\sP$ of the form (\ref{eq:pt_mixture}) is not necessarily Markovian and may not be generated by any SDE.

\paragraph{A stochastic regression objective.} The marginal ODE drift and score expressions in (\ref{eq:sde_for_mixture}) motivate objectives for fitting $\nothat u_t(x)$ and $\nabla\log p_t(x)$ with neural networks when only the \emph{conditional} ODEs and scores are known. Generalizing (\ref{eq:uncond_sf2m}), we define the \emph{(conditional) simulation-free score and flow matching} objective ([SF]$^2$M) for neural networks $v_\theta(\cdot,\cdot)$ approximating the ODE drift and $s_\theta(\cdot,\cdot)$ approximating the score:
\begin{align}\label{eq:cond_sf2m}
    \gL_{\rm [SF]^2M}(\theta)&= \E_{Q^\prime}\underbrace{\|v_\theta(t,x)-\nothat u_t(x|z)\|^2}_{\text{conditional flow matching loss}}\\
    &+\E_{Q^\prime}\lambda(t)^2\underbrace{\|s_\theta(t,x)-\nabla\log p_t(x|z)\|^2}_{\text{conditional score matching loss}},\nonumber
\end{align}
where, as in (\ref{eq:uncond_sf2m}), $\lambda(\cdot)$ is some choice of positive weights and $Q^\prime = (t\sim\gU(0,1)) \otimes q(z) \otimes p_t(x|z)$. This objective can be used to approximate the quantities defined in (\ref{eq:sde_for_mixture}), provided the conditional ODEs and scores are known and $p_t(x|z)$ can be tractably sampled. Correctness is guaranteed by the following Theorem:
\begin{theoremE}[Equality of conditional gradients][end,restate]\label{thm:sf2m_correctness}
    If $p_t(x)>0$ for all $x\in\R^d$ and $t\in[0,1]$, then $\nabla_\theta\gL_{\rm U[SF]^2M}(\theta)=\nabla_\theta\gL_{\rm [SF]^2M}(\theta)$, where $\gL_{\rm U[SF]^2M}(\theta)$ is the unconditional score and flow matching loss (\ref{eq:uncond_sf2m}).
\end{theoremE}
\begin{proofE}
    We show this individually for the flow matching and score matching parts of the losses. (Note that the equality of gradients of the flow matching parts of U[SF]$^2$M and [SF]$^2$M is equivalent to Theorem 3.2 of \citet{tong_conditional_2023}.)

    We claim that for any conditional vector field $\nothat w_t(x|z)$ and $\nothat w_t(x):=\E_{q(z)}\frac{p_t(x|z)}{p_t(x)}\nothat w_t(x|z)$, and approximating vector field $w_\theta(t,x)$, under some regularity conditions on $w_t(x|z)$, we have 
    \begin{equation}\label{eq:w_general_loss}
    \nabla_\theta\E_{z\sim q(z),x\sim p_t(x|z)}\left[\|w_\theta(t,x)-\nothat w_t(x|z)\|^2\right]
    =
    \nabla_\theta\E_{x\sim p_t(x)}\left[\|w_\theta(t,x)-\nothat w_t(x)\|^2\right].
    \end{equation}
    Assuming this claim, the theorem would follow from applying the claim to $\nothat w_t(x|z)=\nothat u_t(x|z)$ and to $\nothat w_t(x|z)=\nabla p_t(x|z)$ for every value of $t$, noting that in these cases $\nothat w_t(x)=\nothat u_t(x)$ and $\nothat w_t(x)=\nabla p_t(x)$, respectively, by \autoref{thm:score_and_flow_of_mixture}, and integrating over $t$.

    We proceed to prove the claim. We drop the distributions in the expectations for conciseness, noting that because $p_t(x)=\E_{z\sim q(z)}p(x|z)$, no ambiguity is caused: the marginal distribution over $x$ that stands under the expectation is the same in $\E_{x\sim p_t(x)}$ and in $\E_{z\sim q(z),x\sim p_t(x|z)}$.
\begin{align*}
&\nabla_\theta\left(\E_{z,x}\left[\|w_\theta(t,x)-\nothat w_t(x|z)\|^2\right]-\E_{x}\left[\|w_\theta(t,x)-\nothat w_t(x)\|^2\right]\right)
\\
=\,&
\nabla_\theta\left(-2\E_{z,x}\langle w_\theta(t,x),\nothat w_t(x|z)\rangle+2\E_x\langle w_\theta(t,x),\nothat w_t(x)\rangle\right)
\end{align*}
where we rewrote the squared norms as inner products and used that $\nothat w_t(x|z),\nothat w_t(x)$ are independent of $\theta$ and that $w_\theta(t,x)$ is independent of $z$. To conclude, we compute
\begin{align*}
\E_{x}\left\langle w_\theta(t,x),\nothat w_t(x)\right\rangle
&=
\iint \left\langle w_\theta(t,x),\nothat w_t(x)\right\rangle p_t(x)\,dx
\\
&=
\int \left\langle w_\theta(t,x),\int\frac{p_t(x|z)}{p_t(x)}\nothat w_t(x|z)q(z)\,dz\right\rangle p_t(x)\,dx\\
&=
\iint\left\langle w_\theta(t,x),\nothat w_t(x|z)\right\rangle p_t(x|z)q(z)\,dx\,dz
\\
&=
\E_{z,x}\left\langle w_\theta(t,x),\nothat w_t(x|z)\right\rangle,
\end{align*}
showing that the difference of gradients vanishes. Note that the above derivation required exchanging the order of integration, which requires some assumptions of regularity at infinity. (Absolute convergence of the integrals is sufficient, and in particular, guaranteed by polynomial growth of in $x$ of $w_\theta$ $\nothat w_t(x|z)$ and exponential decay of $p_t(x|z)$ uniformly in $z$.)
\end{proofE}
This result generalizes Theorem 3.2 of \citet{tong_conditional_2023}. It provides a simulation-free way to train neural networks sufficient to simulate an SDE generating marginals $p_t(x)$ with arbitrary diffusion rate $g(\cdot)$ (cf.\ the discussion following (\ref{eq:uncond_sf2m})). The training and inference algorithms are summarized in \autoref{alg:sf2m-train}, \autoref{alg:sf2m-inference}.

In our approach, the SDE recovered from the ODE and score defined via (\ref{eq:sde_for_mixture}) is the Markovization of the mixture of stochastic processes indexed by $z$. 

\begin{algorithm}[t]
\caption{Score and Flow Matching Training}
\label{alg:sf2m-train}
\begin{algorithmic}
\State {\bfseries Input:} Efficiently samplable $q(z)$, $p_t(x | z)$, computable $u_t(x | z)$, initial networks $v_{\theta}$ and $s_\theta$.
\While{Training}
\State $z \sim q(z); \quad t \sim \mathcal{U}(0, 1)$; \quad $x \sim p_t(x | z)$
\State \hbox{$\mathcal{L}_{\rm [SF]^2M} \gets \| v_\theta(t, x) - \nothat u_t(x | z)\|^2$}\hbox{\quad$+\,\lambda(t)^2 \|s_\theta(t, x) - \nabla_{x}\log p_t(x|z)||^2$}\vskip-1.2em
\Comment {see (\ref{eq:cond_sf2m})}
\State $\theta \gets \mathrm{Update}(\theta, \nabla_\theta \mathcal{L}_{\rm [SF]^2M})$
\EndWhile
\State \Return $v_\theta,s_\theta$
\end{algorithmic}
\end{algorithm}

\paragraph{Sources of conditional ODEs and scores.} 
Although the [SF]$^2$M framework can handle general conditioning information $z$, in this paper we consider the case where $z$ is identified with a pair $(x_0,x_1)$ of a source and target point. For a given $z=(x_0,x_1)$, we will take the conditional probability path $p_t(x|z)$ to be a Brownian bridge with constant diffusion scale $\sigma$, so that $\nothat u_t(x|z)$ and $\nabla\log p_t(x|z)$ are given by (\ref{eq:brownian_bridge_flow}). To avoid numerical issues for $t$ close to 0 or 1, we add a small smoothing constant to the variance. The conditional distributions are thus peaky at $x_0$ and $x_1$ at $t=0$ and $t=1$. (An extension to nonconstant diffusion scale is described in \autoref{sec:nonconstant_sigma}.)

For the resulting marginal $p_t(x)$ to satisfy the boundary conditions $p_0(x)=q_0(x)$ and $p_1(x)=q_1(x)$, $q(x_0,x_1)$ must be a coupling of $q_0$ and  $q_1$ (\ie, a transport plan). This is formalized in the following theorem:
\begin{theoremE}[{[SF]$^2$M} recovers marginals from bridges][end,restate]\label{thm:sf2m_generates_right_marginals}
    If $q(\cdot,\cdot) \in U(q_0, q_1)$ and $v^*_\theta, s^*_\theta$ globally minimize $\mathcal{L}_{\rm [SF]^2M}(\theta)$, the SDE with drift $[v^*_\theta + \frac12 g(t)^2 s^*_\theta]$ and diffusion $g$, and initial conditions $p_0=q_0$, is the Markovization of the mixture of Brownian bridges from $x_0$ to $x_1$ over $q(x_0,x_1)$. In particular, if the SDE generates marginals $p_t$, then $p_1=q_1$.
\end{theoremE}
\begin{proofE}
The probability flow drift $u_t^\circ$ and score $\nabla\log p_t$ of the mixture in question are given by \autoref{thm:score_and_flow_of_mixture} shows. By \autoref{thm:sf2m_correctness}, assuming sufficient regularity, optimization of $\mathcal{L}_{\rm [SF]^2M}(\theta)$ is equivalent to optimization of $\mathcal{L}_{\rm U[SF]^2M}(\theta)$, which is globally minimized when $v_\theta(t,x)=\nothat u_t(x)$ and $s_\theta(t,x)=\nabla\log p_t(x)$ for (Lebesgue,$p_t$)-almost all $t\in [0,1]$ and $x\in\R^d$. Moreover, `almost all' implies `all' if $\nothat u_t,\nabla\log p_t,v_\theta,s_\theta$ are continuous in $t$ and $x$ and $p_t$ has full support.

By hypothesis, the vector field $\nothat u_t(x)$ satisfies the continuity equation jointly with $p_t(x)=\E_{q(z)}p_t(x|z)$, which was assumed to satisfy $p_0(x)=q_0(x)$ and $p_1(x)=q_1(x)$. By the algebraic derivation in \autoref{sec:prelim}, $\nothat u_t(x)+\frac12g(t)^2\nabla\log p_t(x)$ and $p_t(x)$ jointly satisfy the Fokker-Planck equation. Assuming all derivatives appearing in the Fokker-Planck equation exist and are continuous everywhere, the given SDE generates marginal probabilities $p_t$ from initial conditions $p_0$.
\end{proofE}
This theorem tells us that as long as our joint distribution $q(x_0, x_1)$ has the correct marginals, [SF]$^2$M will recover a valid generative model which pushes $q_0$ to $q_1$.%

\subsection{Building Schr\"odinger bridges via {[SF]}$^2$M and entropic optimal transport}\label{sec:solving_sb}
In the previous section, we showed that our method, {[SF]}$^2$M, can approximate the marginal probability $p_t$ of a mixture of processes of the form (\ref{eq:pt_mixture}). In this section, we explain how our {[SF]$^2$M} approximates the SB.

\paragraph{{[SF]}$^2$M approximates the Schr\"odinger bridge.} %
In order to achieve an efficient approximation of the SB, we leverage \autoref{prop:sb_is_mixture_of_bb}. The SB can be expressed as a mixture of Brownian bridges weighted by an entropic optimal transport plan (\ref{eq:bb_mixture}). Therefore, to approximate the SB with {[SF]$^2$M}, we set the distribution $q(x_0, x_1)$ to be equal to the entropic OT plan $\pi^\star_{2 \sigma^2}(q_0, q_1)$ and train the networks $v_\theta$ and $s_\theta$ using \autoref{alg:sf2m-train}. We show that this procedure recovers the SB:
\begin{propositionE}[{[SF]$^2$M} with entropic OT recovers the SB process][end,restate]
    Let $\sP^*$ be the Schr\"odinger bridge between $q_0$ and $q_1$ with respect to $\sQ=\sigma\sW$. If $v^\star_\theta$, $s^\star_\theta$ globally minimize $\mathcal{L}_{\rm [SF]^2M}$, with coupling $\pi^\star_{2\sigma^2}( q_0, q_1)$, then $\sP^*$ is defined by the SDE with drift $[v^*_\theta + \frac12 g(t)^2 s^*_\theta]$, diffusion $g$, and initial conditions $p_0=q_0$.
    \label{prop:sf2M_equal_SB}
\end{propositionE}
\begin{proofE}
This is an immediate consequence of \autoref{thm:sf2m_generates_right_marginals}, which shows that the SDE learned by [SF]$^2$M is the Markovization of a mixture of Brownian bridges, and of \autoref{prop:sb_is_mixture_of_bb}, which characterizes the SB as a mixture of Brownian bridges with mixture weights given by the entropic OT plan.
\end{proofE}

\paragraph{Empirical approximation.} Unfortunately, the real distributions $q_0$ and $q_1$ are usually unknown and we only have access to \emph{i.i.d.}\ samples forming empirical distributions $\hat{q}_0$ and $\hat{q}_1$ of size $n$. Therefore, we can only approximate the true entropic OT plan by computing the entropic OT plan $\pi^\star_{2 \sigma^2}(\hat{q}_0, \hat{q}_1)$ between the empirical distributions~\citep{cuturi_sinkhorn_2013,AltschulerGreenkhorn2017}. This empirical OT plan can be used in {[SF]$^2$M} to construct an \emph{empirical} Schr\"odinger bridge. 

Fortunately, the true entropic OT can be efficiently approximated using empirical distributions, even in high-dimensional spaces \citep{genevay19a, Mena2019}, and it was recently shown that the Schr\"odinger bridge inherits this property \cite[Theorem 5]{stromme_sampling_2023}. In turn, the entropic OT plan between empirical distributions can be efficiently computed using the Sinkhorn algorithm \citep{cuturi_sinkhorn_2013}, which has $\gO(n^2)$  computational complexity \citep{AltschulerGreenkhorn2017}, or using stochastic algorithms \citep{GenevayStocEOT2016, seguy2018large}. However, if this cost is to high (\eg, if $n$ is too large or if one has the true generative process, as in the Gaussian-to-data setting), the plan can be further approximated using minibatch OT \citep{fatras_learning_2020, fatras_unbalanced_2021}; see \autoref{app:ot}. 

The use of an entropic OT plan and marginalization via stochastic regression distinguishes [SF]$^2$M from existing neural SB algorithms (\autoref{tab:algs_comparison}). Such past approaches include mean-matching~\citep[DSB and NLSB,][]{de_bortoli_diffusion_2021,koshizuka_neural_2022}, and bridge-matching approaches~\citep[DSBM and IDBM,][]{shi_diffusion_2023,peluchetti_diffusion_2023}, both of which require an outer iterative proportional fitting loop with an inner training loop. Others have studied the problem assuming paired source and target data \citep[I$^2$SB and ASB,][]{liu_sb_2023,somnath_aligned_2023}; SF$^2$M can be thought of as \emph{inferring} the pairing while jointly fitting the SDE. 

See \autoref{supp:sec:practical_sb} for further discussion of the implications of these choices and practical recommendations.

\begin{table}[]
    \caption{Comparison of SB algorithms (see \autoref{sec:experiments}). {[SF]$^2$M} is the first algorithm that does not assume paired samples, require SDE integration during training, or use an IPF outer loop. DSBM uses simulation only in the outer loop.}\vspace*{-1em}
    \centering
\resizebox{\linewidth}{!}{
     \begin{tabular}{@{}lcccc}
     \toprule
    Algorithm $\rightarrow$ & I$^2$SB/ASB & DSB/NLSB & DSBM/IDBM & {[SF]$^2$M} \\\midrule
       Unpaired samples     & \xmark & \cmark & \cmark & \cmark \\
       Bridge matching      & \cmark & \xmark & \cmark & \cmark \\
       Single loop          & \cmark & \xmark & \xmark & \cmark \\
       Sim.-free training  & \cmark & \xmark & \xmark/\cmark & \cmark \\
       No explicit $(x_0,x_1)$ pairing & \cmark & \cmark & \cmark & \xmark \\
       \bottomrule
    \end{tabular}
    }
    \label{tab:algs_comparison}
\end{table}

\section{LEARNING CELL DYNAMICS WITH {[SF]}$^2$M}\label{sec:alg:single-cell}

Modeling cell dynamics is a major open problem in single-cell data science, as it is important for understanding -- and eventually intervening in -- cellular programs of development and disease~\citep{lahnemann_eleven_2020}. In this section, we show how {[SF]}$^2$M can be used for and adapted to modeling single-cell dynamics.

The cellular dynamics between time-resolved snapshot data, representing observations of cells lying in the space of gene activations, are commonly modeled using Schr\"odinger bridges~\citep{hashimoto_learning_2016,schiebinger_optimal-transport_2019,bunne_recovering_2022,koshizuka_neural_2022}. The applicability of the SB formulation to cell dynamics relies upon the principle of least action, which is thought to hold for cellular systems over short timescales~\citep{schiebinger_reconstructing_2021}, and motivates our choice to apply {[SF]}$^2$M to these problems.%

\begin{figure}%
\centering
\includegraphics[width=0.4\textwidth]{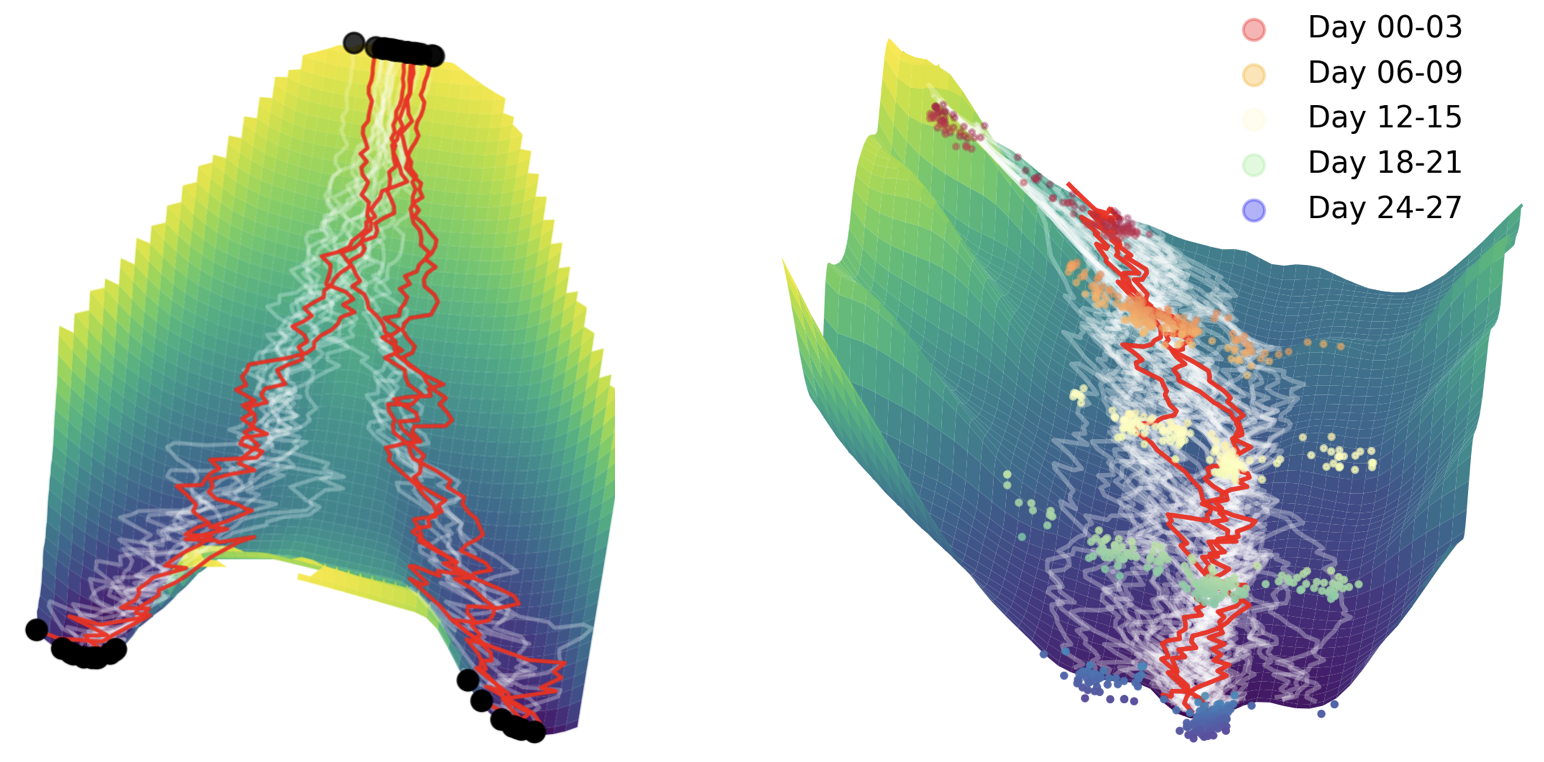}
\caption{Visualization of learned Waddington's landscape $W$ with a bifurcating trajectory for one Gaussian to two Gaussians (\textbf{left}) and for the Embryoid Body (EB) data
~\citep{moon_visualizing_2019} (\textbf{right}). The dimensions are space (left-right), time (forward-back), and potential (up-down). }
\label{fig:landscape}
\end{figure}

\paragraph{Learning flows on cell manifolds.}
Cells are thought to lie on a low-dimensional manifold in the space of gene expressions \citep{moon_manifold_2018}, which has motivated work on density-adhering regularizations~\citep{tong_trajectorynet_2020,koshizuka_neural_2022} and manifold embeddings~\citep{huguet_manifold_2022}. Because {[SF]}$^2$M can use a coupling between marginals $q(x_0,x_1)$ defined by entropic OT with an arbitrary cost function, we can take advantage of these embeddings to compute the pairing using a ground cost that is adapted to the geometry of the manifold.
Specifically, we use the Geodesic Sinkhorn method~\citep{huguet_geodesic_2022}, which computes the entropic OT plan with cost
\begin{equation}
    c_{\rm geo}(x_0, x_1) = \sqrt{- \log \mathcal{H}_t(x_0,x_1)}.
\end{equation}
The matrix $\mathcal{H}_t$ approximates the heat kernel defined via the Laplace-Beltrami operator on the manifold, efficiently approximated using a $k$-nearest-neighbour graph.
We find using this cost leads to more accurate trajectories in high dimensions (see \autoref{tab:sc-dynamics}).

\newcommand{\std}[1]{\scriptsize\rm$\pm$#1}

\begin{table*}[t]
    \caption{Two-dimensional data: generative modeling performance ($\mathcal{W}_2$) and dynamic OT optimality (NPE) divided into SDE methods (top) and ODE methods (bottom). {[SF]}$^2$M performs the best on 3 of 4 datasets and is similar to OT-CFM, which is equivalent to {[SF]}$^2$M as $g(t) \to 0$. *Indicates results taken from \citet{shi_diffusion_2023}.}\vspace*{-1em}
    \label{tab:twodim}
    \centering
    \resizebox{\linewidth}{!}{
\begin{tabular}{@{}lllllllll}
\toprule
Metric $\rightarrow$ & \multicolumn{4}{c}{$\mathcal{W}_2$ ($\downarrow$)} & \multicolumn{4}{c}{Normalized Path Energy ($\downarrow$)} \\\cmidrule(lr){2-5}\cmidrule(lr){6-9}
Algorithm  $\downarrow$ | Dataset $\rightarrow$ &                   ${\gN}\!{\rightarrow}$\ds{8gaussians} &             \ds{moons}$\rightarrow$\ds{8gaussians} &                       ${\gN}\!{\rightarrow}$\ds{moons} &                      ${\gN}\!{\rightarrow}$\ds{scurve} &                   ${\gN}\!{\rightarrow}$\ds{8gaussians} &             \ds{moons}$\rightarrow$\ds{8gaussians} &                       ${\gN}\!{\rightarrow}$\ds{moons} &                      ${\gN}\!{\rightarrow}$\ds{scurve} \\
\midrule
{[SF]}$^2$M-Exact &  \textbf{0.275\std{0.058}}&           0.726\std{0.137}&  \textbf{0.124\std{0.023}}&  \textbf{0.128\std{0.005}}&  \textbf{0.016\std{0.012}}&           0.045\std{0.031}&  \textbf{0.053\std{0.038}}&           0.034\std{0.024}\\
{[SF]}$^2$M-I &  0.393\std{0.054}&           1.482\std{0.151}&  0.185\std{0.028}&  0.201\std{0.062}&  0.160\std{0.019}&           2.577\std{0.323}&  0.855\std{0.130}&           0.845\std{0.106}\\
DSBM-IPF \citep{shi_diffusion_2023}*  & 0.315\std{0.079}& 0.812\std{0.092}& 0.140\std{0.006}& 0.140\std{0.024}& 0.022\std{0.020}& 0.244\std{0.027}& 0.383\std{0.034}& 0.297\std{0.036}\\
DSBM-IMF \citep{shi_diffusion_2023}*  & 0.338\std{0.091}& 0.838\std{0.098}& 0.144\std{0.024}& 0.145\std{0.037}& 0.029\std{0.017}& 0.345\std{0.049}& 0.230\std{0.028}& 0.286\std{0.033}\\
DSB \citep{de_bortoli_diffusion_2021}*        & 0.411\std{0.084}& 0.987\std{0.324}& 0.190\std{0.049}& 0.272\std{0.065}& ---& ---& ---& ---\\
\midrule
OT-CFM \citep{tong_conditional_2023}    &           0.303\std{0.043}&  \textbf{0.601\std{0.027}}&           0.130\std{0.016}&           0.144\std{0.028}&           0.031\std{0.027}&  \textbf{0.015\std{0.010}}&           0.083\std{0.009}&  \textbf{0.027\std{0.012}}\\
SB-CFM  \citep{tong_conditional_2023}    &           2.314\std{2.112}&                         ---&           0.434\std{0.594}&           0.341\std{0.468}&           1.000\std{0.000}&                         ---&           0.995\std{0.000}&           0.745\std{0.039}\\
RF \citep{liu_rectified_2022}       &           0.421\std{0.071}&           1.525\std{0.330}&           0.283\std{0.045}&           0.345\std{0.079}&           0.044\std{0.031}&           0.203\std{0.090}&           0.130\std{0.078}&           0.099\std{0.066}\\
I-CFM \citep{tong_conditional_2023}       &           0.373\std{0.103}&           1.557\std{0.407}&           0.178\std{0.014}&           0.242\std{0.141}&           0.202\std{0.055}&           2.680\std{0.292}&           0.891\std{0.120}&           0.856\std{0.031}\\
FM \citep{lipman_flow_2022}&           0.343\std{0.058}&                         ---&           0.209\std{0.055}&           0.198\std{0.037}&           0.190\std{0.054}&                         ---&           0.762\std{0.099}&           0.743\std{0.116}\\
\bottomrule
\end{tabular}
}
\end{table*}

\paragraph{Learning developmental landscapes.} 
A common model of cell development, known as Waddington's epigenetic landscape \citep{waddington_epigenotype_1942}, assumes that cells evolve and differentiate in the space of gene expressions by following (noisy) gradient ascent on an energy function. 
While a few heuristic methods have been proposed to approximate this energy function from single-cell data before~\citep{tang_potential_2017,qin_single-cell_2023}, we propose a novel approach to model the landscape directly. In our approach, the negative energy can be directly interpreted as an action potential, inspired by the modeling in \citet{neklyudov23a}.

To do this, we impose a Langevin dynamics parametrization on the flow and score in {[SF]}$^2$M: $v_\theta(t, x) = -\nabla_x E_v(t, x)$ and $s_\theta(t, x) = - \nabla_x E_s(t, x)$, where $E_v$ and $E_s$ are neural networks. We can define the Waddington's landscape by $W := E_v + \frac{1}{2} g(t)^2 E_s$. The drift of the SDE is then $u_t(x) = -\nabla_x W(t,x)$, meaning that the time-evolution of a cell follows gradient dynamics on $W$ with added Gaussian noise of scale $g(t)$. We visualize these landscapes in \autoref{fig:landscape} with further details in \autoref{supp:sec:exp_details:landscape}.

\paragraph{Learning gene regulatory networks.} 
Finally, we use {[SF]}$^2$M to learn gene regulatory networks from population snapshots of gene expressions, a persisting challenge in cellular biology \citep{pratapa2020benchmarking}. Following previous work in discovering sparse interaction structure from continuous-time systems~\citep{tank_neural_2021,aliee_beyond_2021,bellot_neural_2022,aliee2022sparsity,atanackovic2023dyngfn}, we define the \emph{gene regulatory network} as the directed graph whose vertices are genes (dimensions of the space) and an edge $i\rightarrow j$ is present if and only if $\frac{\partial (v_\theta(t, x))_j}{\partial x_i} \neq 0$. This directed graph is expected to be sparse.

Previous work resorted to performing inference of trajectories in a low-dimensional (and dense) representation~\citep{tong2023learning,bunne_proximal_2022}, which complicated the discovery of the sparse graph structure in gene space.  {[SF]}$^2$M is the first Schr\"odinger bridge method to scale to high dimensions. %
This allows us to learn a dynamic directly in the gene space and recover the sparse gene interactions. To accomplish this, we use a specialized parametrization of $v_\theta$, inspired by \citet{bellot_neural_2022}, which enables the graph structure to be read out from the sparsity pattern of the initial layer of the trained model (see \autoref{supp:sec:exp_details:grn} for details).

\section{RELATED WORK}\label{sec:rw}
\paragraph{Stochastic continuous-time modeling.}
Our framework is related to both flow-based~\citep{chen_neural_2018,grathwohl_ffjord:_2019,albergo_stochastic_2023,albergo_building_2023,neklyudov_action_2022,liu_rectified_2022} and score-based~\citep{sohl-dickstein_deep_2015,song_generative_2019,song_improved_2020,song_score-based_2021,ho_denoising_2020,winkler_score-based_2023,dhariwal_diffusion_2021,watson_broadly_2022} generative modeling. Both have drawn attention due to their stability and efficiency in training and high quality of generated samples. See \autoref{sec:further_rw} for further discussion.

\paragraph{Schr\"odinger bridge approximation methods.} While there is significant theoretical work on the SB problem~\citep{leonard_schrodinger,stromme_sampling_2023,albergo_stochastic_2023}, practical solutions have assumed paired samples from the Schr\"odinger bridge or required simulation during training. Algorithms based on iterative proportional fitting \citep[DSB and DSBM;][]{de_bortoli_diffusion_2021,shi_diffusion_2023} have the advantage of yielding the exact Schr\"odinger bridge if trained to optimality on each iteration, but may accumulate error with each outer-loop step due to underfitting and function approximation. On the other hand, our proposed {[SF]}$^2$M requires neither training-time integration nor outer-loop iteration and therefore will converge to the exact SB -- if the neural network function class and learning algorithm so allow -- but, unlike DSB and DSBM, require knowledge of the entropic OT plan and the conditional paths (see \autoref{tab:algs_comparison}).

The relative advantages of these algorithms and practical recommendations are further discussed in \autoref{sec:solving_sb}, \autoref{supp:sec:practical_sb}. Furthermore, in \autoref{sec:additional_results} we show that the simulation-based outer loop in \citet{peluchetti_diffusion_2023,shi_diffusion_2023} can be combined with {[SF]}$^2$M to improve the SB marginals at the cost of generative performance.

\paragraph{Applications to cell dynamics.}
When the observer seeks to recover dynamics from multiple snapshots with scRNA-seq data, the machinery of optimal transport can be used \citep{schiebinger_optimal-transport_2019, yang_predicting_2020,tong_trajectorynet_2020,bunne_proximal_2022,huguet_manifold_2022,bunne_recovering_2022,koshizuka_neural_2022}. However, these methods all require simulation during training, which scales poorly to high dimensions.

\section{EXPERIMENTS}\label{sec:experiments}

In this section we empirically evaluate {[SF]}$^2$M with respect to optimal transport, generative modeling, and single-cell interpolation criteria. We compare:
\begin{itemize}[nosep,left=0pt]
\item Minibatch {[SF]}$^2$M with exact OT minibatches ({[SF]}$^2$M-Exact), with entropic OT (Sinkhorn) minibatches (-Sink), with independent couplings (-I), and with Geodesic OT (-Geo) when applicable.
\item A variety of (ODE) flow-based models, including optimal transport conditional flow matching \citep[OT-CFM,][]{tong_conditional_2023}, rectified flow \citep[RF,][]{liu_rectified_2022}, and flow matching  \citep[FM,][]{lipman_flow_2022}.
\item Schr\"odinger bridge models: diffusion Schr\"odinger bridges  \citep[DSB,][]{de_bortoli_diffusion_2021} and diffusion Schr\"odinger bridge matching \citep[DSBM,][]{shi_diffusion_2023}, which is equivalent to work on iterated diffusion mixture transport \citep[IDBM,][]{peluchetti_diffusion_2023}. 
\item Single-cell dynamics models: Neural Lagrangian Schr\"odinger bridges \citep[NLSB,][]{koshizuka_neural_2022}, TrajectoryNet \citep{tong_trajectorynet_2020}.
\end{itemize}
See \autoref{supp:sec:exp_details} for all experiment details. All results are presented as mean $\pm$ std.\ over five seeds.

\begin{table}[t]
    \caption{Gaussian-to-Gaussian Schr\"odinger bridges with $10^4$ datapoints between a Gaussian with parameters estimated from empirical samples ($p_t$) with error to the continuous Schr\"odinger bridge marginals ($p_t^*$) either at the target distribution (left) or averaged across 21 timepoints (right).}\vspace*{-1em}
    \label{tab:gaussians}
    \centering
    \resizebox{\linewidth}{!}{
\begin{tabular}{@{}lllllll}
\toprule
Metric $\rightarrow$&\multicolumn{3}{c}{${\rm KL}(p_1, p_1^*)$}&\multicolumn{3}{c}{Mean ${\rm KL}(p_t, p_t^*)$} \\\cmidrule(lr){2-4}\cmidrule(lr){5-7}
Alg.\ $\downarrow$ | Dim. $\rightarrow$ &\multicolumn1c5&\multicolumn1c{20}&\multicolumn1c{50}&\multicolumn1c5&\multicolumn1c{20}&\multicolumn1c{50}\\
\midrule
{[SF]}$^2$M-Exact&0.007\std{0.000}&\textbf{0.029\std{0.002}}&\textbf{0.124\std{0.003}}&0.006\std{0.000}&\textbf{0.028\std{0.001}}&0.258\std{0.001}\\
DSBM-IPF&0.015\std{0.005}&0.132\std{0.004}&0.528\std{0.013}&\textbf{0.005\std{0.002}}&0.050\std{0.002}&\textbf{0.221\std{0.004}}\\
DSB&8.757\std{}---&49.963\std{}---&221.213\std{}---&8.757\std{}---&49.963\std{}---&221.213\std{}---\\
\midrule
SB-CFM&\textbf{0.001\std{0.000}}&0.034\std{0.003}&0.170\std{0.002}&0.008\std{0.000}&0.086\std{0.002}&0.447\std{0.003} \\
\bottomrule
\end{tabular}
}
\end{table}

\begin{table}[t]
\centering
\begin{minipage}[c]{\columnwidth}
\caption{Single-cell comparison over three datasets, averaged over leaving out different intermediate timepoints on 5 PCs. For each left-out point, we measure the 1-Wasserstein distance between the imputed and ground truth  distributions at the left-out time point, following \citet{tong_trajectorynet_2020}. *Indicates values taken from aforementioned work.}\vspace*{-1em}
\label{tab:sc-dynamics-fived}
\end{minipage}
\begin{minipage}[c]{\columnwidth}
\resizebox{1\linewidth}{!}
{
\begin{tabular}{@{}llll}
\toprule
Algorithm  $\downarrow$ | Dataset $\rightarrow$ &                       \multicolumn1c{Cite}&                          \multicolumn1c{EB} &                      \multicolumn1c{Multi}\\
\midrule
{[SF]}$^2$M-Geo&1.017\std{0.104}&0.879\std{0.148}&1.255\std{0.179}\\
{[SF]}$^2$M-Exact&\textbf{0.920\std{0.049}}&\textbf{0.793\std{0.066}}&\textbf{0.933\std{0.054}} \\
{[SF]}$^2$M-Sink&1.054\std{0.087}&1.198\std{0.342}&1.098\std{0.308} \\
DSBM \citep{shi_diffusion_2023}&1.705\std{0.160}&1.775\std{0.429}&1.873\std{0.631} \\
DSB \citep{de_bortoli_diffusion_2021}&0.953\std{0.140}&0.862\std{0.023}&1.079\std{0.117} \\
\midrule
OT-CFM \citep{tong_conditional_2023}&\textbf{0.882\std{0.058}}&\textbf{0.790\std{0.068}}&\textbf{0.937\std{0.054}} \\
I-CFM \citep{tong_conditional_2023}&0.965\std{0.111}&0.872\std{0.087}&1.085\std{0.099} \\
SB-CFM \citep{tong_conditional_2023}&1.067\std{0.107}&1.221\std{0.380}&1.129\std{0.363} \\
\midrule
Reg. CNF \citep{finlay_how_2020}*&---&0.825&--- \\
TrajectoryNet \citep{tong_trajectorynet_2020}*&---&0.848&--- \\
NLSB \citep{koshizuka_neural_2022} & --- & 0.970 & ---\\
\bottomrule
\end{tabular}
}
\end{minipage}
\end{table}

\begin{table}[t]
 \caption{Leave-one-timepoint-out testing of dynamics interpolation methods measuring the error between the predicted and ground truth left out timepoint using the 1-Wasserstein distance. We test on 50 and 100 principal components as well as 1000 highly variable genes.}\vspace*{-1em}
    \label{tab:sc-dynamics}
    \centering
    \resizebox{\linewidth}{!}{
\begin{tabular}{@{}llllllll}
\toprule
 Dim. $\rightarrow$ &  \multicolumn{2}{c}{50} & \multicolumn{2}{c}{100} & \multicolumn{2}{c}{1000}\\
 \cmidrule(lr){2-3}\cmidrule(lr){4-5}\cmidrule(lr){6-7}
 Alg.\ $\downarrow$ | Dataset $\rightarrow$ &                         \multicolumn1c{Cite} &                                                \multicolumn1c{Multi} &                         \multicolumn1c{Cite} &                                                 \multicolumn1c{Multi}&                        \multicolumn1c{Cite}&                                               \multicolumn1c{Multi}\\
\midrule
{[SF]}$^2$M-Geo                       & \textbf{38.52\std{0.29}} &    \textbf{44.80\std{1.91}} &  \textbf{44.50\std{0.42}} &  \textbf{52.20\std{1.96}} & \textbf{40.09\std{1.53}} &  \textbf{51.29\std{0.09}}\\
{[SF]}$^2$M-Exact                     & 40.01\std{0.78} &                  45.34\std{2.83} &           46.53\std{0.43} &                    52.89\std{1.99} & 43.66\std{0.72} &           53.15\std{1.86}\\
DSBM        & 53.81\std{7.74}  &  66.43\std{14.39} &  58.99\std{7.62} &    70.75\std{14.03} & 50.09\std{4.81} &          61.71\std{13.90}\\
\midrule
OT-CFM   & 38.76\std{0.40} &             47.58\std{6.62} &           45.39\std{0.42} &            54.81\std{5.86} & 43.25\std{0.73} &           52.29\std{1.55} \\
I-CFM    & 41.83\std{3.28} &    49.78\std{4.43} &  48.28\std{3.28} &  57.26\std{3.86} & 44.12\std{0.52} &           52.99\std{1.50} \\
\bottomrule
\end{tabular}
}
\end{table}

\paragraph{{[SF]}$^2$M is a competitive generative model for low-dimensional data.}
We first evaluate in \autoref{tab:twodim} how well various methods approximate dynamic optimal transport on low-dimensional datasets (\ds{8gaussians}, \ds{moons}, and \ds{scurve}). We train Schr\"odinger bridges between a Gaussian and each dataset, and between \ds{8gaussians} and \ds{moons}, using [SF]$^2$M. We report the 2-Wasserstein distance between the predicted distribution and the target distribution with samples of size 10,000. Following \citet{tong_conditional_2023}, we also report the Normalized Path Energy relative to the 2-Wasserstein distance, defined as ${\rm NPE}(p,q) := |\int \|v_\theta(t, x)\|^2 dt - \mathcal{W}_2^2(p, q) | / \mathcal{W}_2^2(p, q)$. This metric is equal to zero if and only if $v_\theta$ solves the dynamic optimal transport problem. %
\autoref{tab:twodim} summarizes our results, showing that {[SF]}$^2$M outperforms all methods, both stochastic (top) and deterministic (bottom). Despite minibatch OT bias (which can be seen as a form of regularization like the entropic regularization \citep{fatras_learning_2020}), we find {[SF]}$^2$M-Exact approximates the Schr\"odinger bridge best, with the OT computation accounting for only 1\% of the training time on batch sizes of 512. 

\paragraph{{[SF]}$^2$M recovers the SB.}
Next we evaluate how well [SF]$^2$M can model Schr\"odinger bridge marginals. We use a Gaussian-to-Gaussian Schr\"odinger bridge because it has closed-form Gaussian marginals~\citep{mallasto_entropy_2021,bunne_recovering_2022} following \citet{de_bortoli_diffusion_2021}. After training all methods, we evaluate the quality of empirical marginals with respect to the ground truth by sampling trajectories using \autoref{alg:sf2m-inference}. We compute the KL divergence between a Gaussian approximation of the  empirical marginal and the Gaussian marginal of the ground truth Schr\"odinger bridge at multiple timepoints. This evaluation is shown in \autoref{tab:gaussians} at just the last timepoint ($t=1$) and an average over 21 equally spaced timepoints. We train each method for an equal number of steps, using 20 outer loops for DSB and DSBM, which require iterative optimization of forward and backward models. We find that in low dimensions SB-CFM, which corresponds to [SF]$^2$M's probability ODE flow, performs the best, closely followed by {[SF]}$^2$M-Exact. In high dimensions, {[SF]}$^2$M-Exact better matches the target distribution, and performs similarly to DSBM and significantly better than DSB on the intermediate marginals.

\begin{figure}[tb]
     \centering
     \includegraphics[width=\columnwidth]{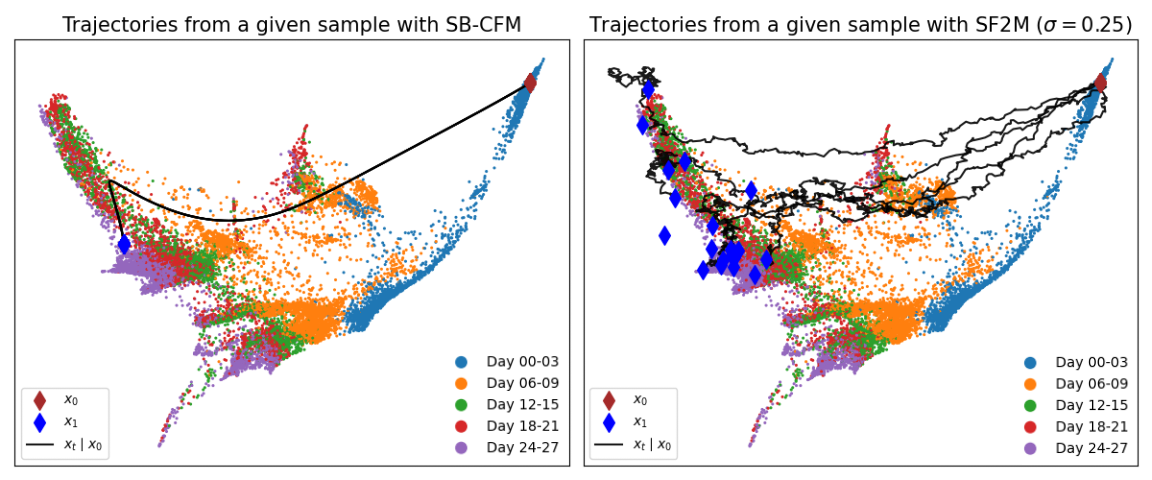}
      \vspace{-1.5em}
    \caption{Simulation of trajectories from a given cell on 2D EB data. \textbf{Left:} Probability flow ODE trajectory, approximated by SB-CFM \citep{tong_conditional_2023}. \textbf{Right:} Five SDE trajectories from [SF]$^2$M; more target samples (20) in blue.} 
    \label{fig:condi_single_cell}
\end{figure}

\paragraph{{[SF]}$^2$M accurately models high-dimensional single cell dynamics.}
We train our method {[SF]}$^2$M on single cell dynamics, as described in \autoref{sec:alg:single-cell}, on three real-world datasets in the setup established by \citet{tong_trajectorynet_2020} (see \autoref{supp:sec:exp_details:landscape}) and gathered our results for different dimensions in \autoref{tab:sc-dynamics-fived} and \autoref{tab:sc-dynamics}. Given $K$ unpaired data distributions representing a cell population at $K$ different timepoints, we solve a SB problem between every two successive time points, sharing parameters between the models. To test the interpolation ability of the trained models, we perform leave-one-out interpolation, predicting timepoint $k$ using a model trained on timepoints $[1, \ldots, k-1, k+1, \ldots, K]$.  
We consider four data representations of different dimensionality: using the first 5, 50, or 100 whitened principal components and using the 1000 dimensions corresponding to the most highly variable genes \citep{wolf_scanpy_2018}. [SF]$^2$M performs  the best among the Schr\"odinger bridge methods and similarly to the ODE-based OT-CFM on the low-dimensional data (\autoref{tab:sc-dynamics-fived}), but is better in higher dimensions (\autoref{tab:sc-dynamics}). As the number of dimensions grows, the advantage of geodesic interpolation ({[SF]}$^2$M-Geo) becomes apparent.

We also compare the stochastic process modeled by {[SF]}$^2$M to its probability flow ODE SB-CFM. We show in \autoref{fig:condi_single_cell} that [SF]$^2$M models a stochastic dependence of the output on the input, unlike SB-CFM, despite the two algorithms sharing marginal densities. This is important in the EB data as the initial stem cells are thought to be pluripotent and should evolve stochastically into differentiated cell types over time. Such differentiation cannot be modelled by ODE-based methods, thus motivating the use of a stochastic process for modeling single-cell dynamics.

\paragraph{{[SF]}$^2$M can be used to recover gene regulatory networks.}
We demonstrate the use-case of {[SF]}$^2$M for recovering gene regulatory networks (GRNs) from single-cell gene expression using the algorithm described at the end of \autoref{sec:alg:single-cell}. 
We show how we can use {[SF]}$^2$M to simultaneously learn dynamics and GRN structure from single-cell gene expression data. 
We use BoolODE \citep{pratapa2020benchmarking} to simulate two single-cell systems given ground truth GRNs: (1) a system with bifurcating trajectories (7 genes), and (2) a system with trifurcating trajectories (9 genes).
We summarize our results in \autoref{tab:boolode-grn}. We also measure how accurately the ground truth GRN is recovered using the standard AUC-ROC and average precision (AP) metrics. 
We find that {[SF]}$^2$M-Exact with no noise (corresponding to OT-CFM \citep{tong_conditional_2023}) performs best at inferring the underlying GRN as compared to correlation baselines (Pearson and Spearman correlation), a mutual information baseline DREMI~\citep{krishnaswamy_conditional_2014}, and pairwise Granger causality~\citep{granger_investigating_1969}. Higher values of $\sigma$ do not perform as well, but still outperform the baselines on most metrics. 

\begin{table}[t]
    \caption{GRN recovery from simulated time-lapsed single-cell gene expression. Shows structure predictive performance in terms of area under the receiver operator characteristic (AUC-ROC) and average precision (AP).}\vspace*{-1em}
    \label{tab:boolode-grn}
    \centering
    \resizebox{\linewidth}{!}{
\begin{tabular}{@{}lllll}
\toprule
GRN $\rightarrow$ & \multicolumn{2}{c}{Bifurcating System} & \multicolumn{2}{c}{Trifurcating System} \\
\cmidrule(lr){2-3}\cmidrule(lr){4-5}
    Alg.\ $\downarrow$ | Metric $\rightarrow$ &  AUC-ROC ($\uparrow$) & AP ($\uparrow$) &  AUC-ROC ($\uparrow$) & AP ($\uparrow$)\\
\midrule
    NGM-{[SF]}$^2$M$_{\sigma=0}$ & \textbf{0.786}\std{0.081} & \textbf{0.521}\std{0.160} &  \textbf{0.764}\std{0.066} & \textbf{0.485}\std{0.105} \\
    NGM-{[SF]}$^2$M$_{\sigma=0.1}$ & 0.723\std{0.014} & 0.444\std{0.030} & 0.731\std{0.077} & 0.453\std{0.091} \\
\midrule
    Spearman & 0.755\std{0.003} & 0.438\std{0.002} & 0.718\std{0.00} & 0.413\std{0.005} \\
    Pearson  & 0.744\std{0.000} & 0.415\std{0.000} & 0.710\std{0.00} & 0.405\std{0.002} \\
    DREMI    & 0.594\std{0.017} & 0.293\std{0.011} & 0.419\std{0.02} & 0.205\std{0.007} \\
    Granger  & 0.664\std{0.013} & 0.421\std{0.04 } & 0.613\std{0.04} & 0.343\std{0.039} \\
    SCODE    & 0.570\std{0.036} & 0.370\std{0.028} & 0.570\std{0.06} & 0.332\std{0.077} \\
    
\bottomrule
\end{tabular}
}
\end{table}

\section{CONCLUSION}
We have introduced a novel class of simulation-free objectives for learning continuous-time stochastic generative models between general source and target distributions. For sources and targets with finite support, we can directly approximate the continuous-time Schr\"odinger bridge without simulation by computing the entropic OT plan via efficient algorithms. We have shown how our method can be applied to learn cell dynamics and extract the gene regulatory structure. Future work can consider how to train {[SF]}$^2$M-like models with interventional data to improve GRN inference. 

\paragraph{Limitations.} The main limitation of {[SF]}$^2$M is that it requires knowledge of conditional path distributions (Brownian bridges). These distributions are not available in closed form if one considers more general reference processes~\citep{fernandes_shooting_2021}, which may be useful to encode biological priors~\citep{koshizuka_neural_2022}, or on general Riemannian manifolds. 

\section*{Acknowledgments}

We would like to thank Hananeh Aliee, Paul Bertin, Valentin de Bortoli, Stefano Massaroli, and Austin Stromme for productive conversations.
The authors acknowledge funding from CIFAR, Genentech, Samsung, IBM, Microsoft, and Google. We are also grateful to the anonymous reviewers for suggesting numerous improvements. This research was enabled in part by compute resources provided by Mila (mila.quebec) and NVIDIA Corporation. In addition, K.F.\ acknowledges funding from NSERC (RGPIN-2019-06512) and G.W.\ acknowledges funding from NSERC Discovery grant 03267 and NIH grant R01GM135929.

\bibliography{tidy}
\bibliographystyle{apalike}

\clearpage

\section*{Checklist}

 \begin{enumerate}

 \item For all models and algorithms presented, check if you include:
 \begin{enumerate}
   \item A clear description of the mathematical setting, assumptions, algorithm, and/or model. [Yes]
   \item An analysis of the properties and complexity (time, space, sample size) of any algorithm. [Yes]
   \item Source code, with specification of all dependencies, including external libraries. [Yes]
 \end{enumerate}

 \item For any theoretical claim, check if you include:
 \begin{enumerate}
   \item Statements of the full set of assumptions of all theoretical results. [Yes]
   \item Complete proofs of all theoretical results. [Yes]
   \item Clear explanations of any assumptions. [Yes]     
 \end{enumerate}

 \item For all figures and tables that present empirical results, check if you include:
 \begin{enumerate}
   \item The code, data, and instructions needed to reproduce the main experimental results (either in the supplemental material or as a URL). [Yes]
   \item All the training details (e.g., data splits, hyperparameters, how they were chosen). [Yes]
         \item A clear definition of the specific measure or statistics and error bars (e.g., with respect to the random seed after running experiments multiple times). [Yes]
         \item A description of the computing infrastructure used. (e.g., type of GPUs, internal cluster, or cloud provider). [Yes]
 \end{enumerate}

 \item If you are using existing assets (e.g., code, data, models) or curating/releasing new assets, check if you include:
 \begin{enumerate}
   \item Citations of the creator if your work uses existing assets. [Yes]
   \item The license information of the assets, if applicable. [Not Applicable]
   \item New assets either in the supplemental material or as a URL, if applicable. [Not Applicable]
   \item Information about consent from data providers/curators. [Not Applicable]
   \item Discussion of sensible content if applicable, e.g., personally identifiable information or offensive content. [Not Applicable]
 \end{enumerate}

 \item If you used crowdsourcing or conducted research with human subjects, check if you include:
 \begin{enumerate}
   \item The full text of instructions given to participants and screenshots. [Not Applicable]
   \item Descriptions of potential participant risks, with links to Institutional Review Board (IRB) approvals if applicable. [Not Applicable]
   \item The estimated hourly wage paid to participants and the total amount spent on participant compensation. [Not Applicable]
 \end{enumerate}

 \end{enumerate}
 
\clearpage
\appendix
\onecolumn
\aistatstitle{Simulation-Free Schr\"odinger Bridges via Score and Flow Matching: \\
Supplementary Materials}
\vspace{-6cm}

Our code can be found at \url{https://github.com/atong01/conditional-flow-matching}. The supplementary material is structured as follows:
\begin{itemize}
    \item Appendix \ref{app:ot} gives more background on optimal transport.
    \item Appendix \ref{sec:proofs} presents the proofs of our different results.
    \item Appendix \ref{sec:further_rw} gives more background on the related work.
    \item Appendix \ref{sec:additional_results} describes additional results and experiments.
    \item Appendix \ref{sec:nonconstant_sigma} discusses Schr\"odinger bridges with varying diffusion rate.
    \item Appendix \ref{supp:sec:exp_details} presents the experimental details of our experiments in the main paper.
\end{itemize}

\section{BACKGROUND ON OPTIMAL TRANSPORT}
\label{app:ot}

In this section, we review optimal transport and its application in machine learning.

\begin{algorithm}
\caption{Simulation-Free Score and Flow Matching Inference (with Euler-Maruyama integration)}
\label{alg:sf2m-inference}
\begin{algorithmic}
\State {\bfseries Input:} Source distribution $q_0$, flow and score networks $v_\theta$ and $s_\theta$, diffusion schedule $g(\cdot)$, integration step size $\Delta t$.
\State $x_0 \sim q_0(x)$
\For{$t$ in $[0, 1 / \Delta t)$}
\State $u_t \gets v_\theta(t, x_t) + \frac{g(t)^2}{2}s_\theta(t, x_t)$
\State $x_{t+\Delta t} \sim \mathcal{N}(x + u_t \Delta t, g(t)^2 \Delta t)$
\EndFor
\State \Return Samples $x_1$
\end{algorithmic}
\end{algorithm}

\begin{figure}[b!]
     \centering
     \includegraphics[width=\columnwidth]{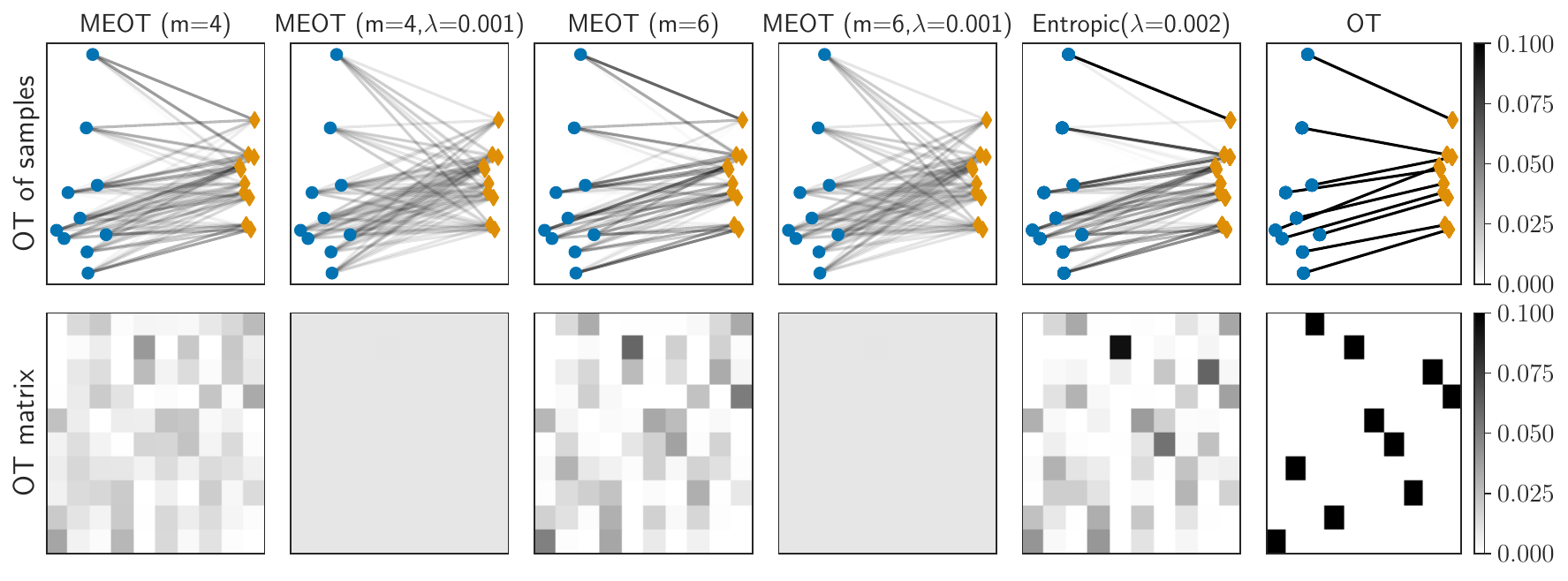}
      \vspace{-0.2cm}
    \caption{Optimal transport couplings for different OT costs and batch sizes on a 2D example. The top row represent the OT matching between samples while the bottom row represent the minibatch OT plan. We can see that coupling entropic OT with minibatches lead to a uniform plan contrary to using only entropic regularization or minibatch approximation.} 
    \label{fig:mbot_plan}
\end{figure}

\subsection{Minibatch OT}\label{supp:sec:mbot}

In the context of generative modeling, the source distribution is a Gaussian distribution and the target distribution is the real data distribution. This scenario corresponds to a semi-discrete optimal transport problem. Therefore, the Sinkhorn algorithm cannot be used to compute the entropic OT plan between distributions. It is nonetheless possible to compute it with stochastic algorithms \citep{GenevayStocEOT2016}. Unfortunately, these stochastic algorithms are slow to converge and it is prohibitive for large scale datasets. Therefore, we chose to rely on a minibatch optimal transport approximation \citep{fatras_learning_2020, fatras_minibatch_2021}. Minibatch OT computes the OT between minibatches of samples and thus corresponds to the discrete-discrete optimal transport setting. It is known to have quadratic computational and memory costs in the number of samples (see \autoref{supp:sec:complexity} for a longer discussion). 

The minibatch optimal transport approximation has been successfully before used in generative modeling \citep{genevay18a, salimans2018improving}. As it is different from the original optimal transport problem, we want to make sure that the basic properties from Schr\"odinger Bridges are conserved. We use the minibatch optimal transport plans definitions from \citet{fatras_learning_2020}, namely

\begin{definition}[minibatch transport plan \citep{fatras_learning_2020}]
Consider $\an$ and $\bn$ two empirical probability distributions. For each $A=\{a_1, \dots, a_m\} \in \mathcal{P}_m(\an)$ and $B=\{b_1, \dots, b_m\} \in \mathcal{P}_m(\bn)$ we denote by $\Pi_{ A, B }$ the optimal plan between the random variables, considered as a $n \times n$ matrix where all entries are zero except those indexed in $A \times B$. We define the \textit{averaged minibatch transport matrix}:
\begin{equation}
 \Pi^m(\an, \bn)  = \dbinom{n}{m}^{-2} \sum_{A \in \Pa} \sum_{B \in \Pb} \Pi_{  A, B }. \label{eq:pim}
\end{equation}
Following the subsampling idea, we define the subsampled minibatch transportation matrix:
\begin{equation}
 \Pi^k(\an, \bn) :=  k^{-1} \sum_{  (A, B)  \in D_k  }  \Pi_{  A, B }
\end{equation}
where $D_k$ is a set of cardinality $k$ whose elements are drawn at random from the uniform distribution on $ \Gamma:= \mathcal{P}_m( \{x_0^1, \cdots, x_0^n  \}) \times \mathcal{P}_m( \{x_1^1, \cdots, x_1^n \} )  $.
\label{def:mbot}
\end{definition}

Note that $\Pi_k$ converges exponentially fast to $\Pi_m$ as $k$ grows \cite[Theorem 2]{fatras_learning_2020}. In practice, it is enough to set $k$ equal to 1 to get good performance in deep learning applications \citep{genevay18a, Damodarandeepjdot2018, fatras_unbalanced_2021}. Therefore, the subsampling estimator does not have the correct marginals in general, contrary to $\Pi_m$ which has always the right marginals \cite[Proposition 1]{fatras_learning_2020}.  Minibatch optimal transport has been shown to implicitly regularize the transport plan \citep{fatras_learning_2020, fatras_minibatch_2021}. Indeed, as we draw uniformly at random sample to build the minibatches, we create non optimal connections. This is similar to the entropic optimal transport which densifies the transport plan. Therefore, coupling the minibatch approximation with the entropic OT cost might lead to an extremely dense plan that is close to the uniform transport plan. Unfortunately such a transport plan looses all geometric insights from data and is also far from the original entropic OT cost. We illustrate this phenomenon in \autoref{fig:mbot_plan} on a toy example between two 2D data distributions. Notably, the minibatch OT plan is closer to the entropic OT plan than the minibatch entropic OT plan. That is why in our experiments, we observed that the minibatch approximation with exact optimal transport outperforms the minibatch approximation with the entropic OT cost. We leave the question of closeness between minibatch OT and entropic OT as future work. While it would have been possible to decrease these non-optimal connections with OT variants \citep{fatras_unbalanced_2021}, it would have added extra hyperparameters that would have led to more compute. 

\subsection{Computational complexity}\label{supp:sec:complexity}
The batch size we use is 512 for all experiments except for the Gaussian-to-Gaussian experiments, which use batch size 500 to match the DSBM baseline. The complexity of static discrete OT is cubic in the batch size $m$ and linear in the dimension $d$: $O(m^3 + m^2d)$. This can be reduced to quadratic complexity in the entropy-regularized discrete OT case. However, for batches <10k, the exact minibatch OT is often actually faster than the Sinkhorn solver for the entropic OT problem thanks to the optimized simplex solver code in POT \citep{flamary_pot_2021}. Note that the memory complexity is quadratic in the number of samples due to the storage of the ground cost matrix.

The asymptotic complexity does not tell the full story, as these discrete algorithms run much faster than the (theoretically linear time) neural network training. In practice, for high dimensional settings we find a $<1\%$ overhead in including the OT solution per batch (much of which is in the transfer from GPU to CPU and back). We note that this is significantly faster than contemporary works such as DSBM and IBDM, which must train a bridge matching model for every iterative proportional fitting (IPF) step.

\section{PROOFS OF THEOREMS}\label{sec:proofs}
\printProofs

\section{FURTHER DISCUSSION ON RELATED WORK}
\label{sec:further_rw}

Next we discuss similarities and differences in related algorithms, including deterministic flow models and stochastic score-based generative models. 

\subsection{Flow models}

Recently, there has been significant advances in simulation-free training of flow models, which were originally trained using maximum likelihood training in what is termed as continuous normalizing flows~\citep{chen_neural_2018,grathwohl_ffjord:_2019}. In this section, we discuss the more recent flow matching techniques, which allow for simulation-free training of flow models. In this paper we show that any deterministic flow model can be converted into a stochastic model with the addition of a conditional score matching loss. This generalizes flow matching ideas to SDEs and provides a simple link between score-based generative modelling and flow-based generative modeling.

Conditional flow matching, terminology introduced by \citet{lipman_flow_2022}, and extended to dynamic optimal transport in \citet{tong_conditional_2023,pooladian_2023_multisample}, trains with the \textit{conditional flow matching} loss,
\begin{equation}
    \mathcal{L}_{\text{CFM}}(\theta) = \E_{t \sim \mathcal{U}(0,1), z\sim q(z), x \sim p_t(x | z)} \| v_\theta(t, x) - \nothat{u}_t(x | z)\|^2.
\end{equation}
for some predefined conditional paths $z, q(z), u_t(x | z)$, and $p_t(x | z)$. Depending on the choice of conditioning and probability paths we can recover most known flow matching techniques, such as the originally described flow matching~\citep{lipman_flow_2022}, action matching~\citep{neklyudov_action_2022}, stochastic interpolants~\citep{albergo_building_2023,albergo_stochastic_2023}, the 1-rectified flow~\citep{liu_rectified_2022}, optimal transport conditional flow matching (OT-CFM)~\citep{tong_conditional_2023} also known as multisample flow matching~\citep{pooladian_2023_multisample}, and its deterministic Schr\"odinger bridge counterpart, Schr\"odinger bridge conditional flow matching (SB-CFM)~\citep{tong_conditional_2023}.

\paragraph{SDEs vs.\ ODEs.} Flow models are a powerful way to learn deterministic dynamics which are often faster to sample from~\citet{song_denoising_2021}, particularly with ideas from dynamic optimal transport~\citep{tong_conditional_2023,pooladian_2023_multisample}. However, recent work \citep{liu_sb_2023,shi_diffusion_2023} has noted advantages of stochastic dynamics which we also observe (see \autoref{tab:twodim} and \autoref{tab:gaussians}), particularly in high dimensional settings and in terms of generative performance. In this work, we also seek to model an inherently stochastic system, where there is randomness introduced based on a variety of external factors. In particular, the fact that a single-cell develops into a whole organism, and similarly, the fact that a single population of stem cells develops into a multitude of different cell types, requires the modelling with stochastic dynamics that can model complex conditional distributions. While it is possible to approximate the conditional distribution with lineage-tracing techniques~\citep{kester_single-cell_2018,wagner_lineage_2020,klein_mapping_2023}, these are biologically complex, and we leave their analysis to future work.

\subsection{Learning Schr\"odinger bridges}

Schr\"odinger bridge models have been used as the source-conditioned variation of score-based generative models, which can be conditioned on a variety of source distributions including dirac~\citep{wang_deep_2021}, and from data or noised data~\citep{somnath_aligned_2023,liu_sb_2023}. Most algorithms are based on mean-matching, which is simulation-based and requires simulating, storing, and matching whole trajectories~\citep{de_bortoli_diffusion_2021}, with similar algorithms introduced by \citet{vargas_solving_2021,chen_likelihood_2022}. In particular, these algorithms parameterize a forward drift $v_\theta^f$ and backwards drift $v_\theta^b$, simulating trajectories in one direction and training the reverse direction to match them. Recent work extends this to Markovian bridges (termed bridge matching)~\citep{shi_diffusion_2023,peluchetti_diffusion_2023}, which greatly improves performance.

\subsection{Practical implications of Schr\"odinger bridge modeling choices}\label{supp:sec:practical_sb}

\paragraph{Approximation error.}
All known algorithms for learning Schr\"odinger bridges for general distributions can only approximate the Schr\"odinger bridge. Approximation error accumulates from a number of sources, which can roughly be categorized into the following areas:
\begin{enumerate}[left=0pt,nosep]
    \item Error in the static joint optimization $\hat{\pi}(x_0, x_1)$ approximating the true joint $\pi^\star(x_0, x_1)$. Since there is no closed form for this optimization, all approximations are iterative and discrete. Error in this distribution accumulates in two places, and can be divided into error on the marginals $|\hat{\pi}(x_0, \cdot) - q(x_0)|$, $|\hat{\pi}(\cdot, x_1) - q(x_1)|$, and error on the joint $|\hat{\pi}(x_0, x_1) - \pi^\star(x_0, x_1)|$:
        \begin{enumerate}
            \item Marginal and Joint error from using a finite number of iterations (IPF underfitting). IPF both in continuous space (as used by mean-matching and bridge matching) and IPF in discrete space (as used by [SF]$^2$M) are run for a finite number of iterations. Continuous space IPF is much more expensive (involving fitting and simulating neural networks) and is thus is used with tens of iterations, often $L=20$ \citep{de_bortoli_diffusion_2021, shi_conditional_2022}. In contrast, discrete space IPF uses many more iterations. We use the default python optimal transport (POT) parameters with $L_{max} = 1000$ with early stopping if the marginal error is $< 10^{-9}$~\citep{flamary_pot_2021}. The discrete OT computation adds a negligible ($<1$\%) computational overhead in [SF]$^2$M.
            \item Joint error from minibatch approximation (discrete IPF only). By using discrete OT solvers, we accumulate discretization error between by using finite batch sizes.  This affects the recovery of the true Schr\"odinger bridge, but does not affect the marginal error, which is the important error for generative modeling performance at the endpoints.  As batch size $\to \infty$, this error goes to zero; for a finite dataset of small enough size, [SF]$^2$M can use full-batch discrete entropic OT, where the batch size used for OT can be larger than that used for neural network training.
        \end{enumerate}
        In practice, [SF]$^2$M, by using discrete OT solvers, trades off increased error in the joint via minibatch approximation error for reduced error from IPF underfitting. It is difficult to bound these errors theoretically, however, our experiments show that this tradeoff is useful in practice. We leave further theoretical investigation to future work.
    \item Finite sample error. Datasets are often finite (\ie, discrete). Approximating the continuous densities from discrete samples is challenging and accumulates error. 
    \item Neural network fitting approximation (applicable to all neural SDE approximations to Schr\"odinger bridges).
    \item Discretization error. %
    Simulation-free objectives like [SF]$^2$M treat time as a continuous variable, but, at inference time, error accumulates from integration in discrete time (\eg, using an Euler-Maruyama scheme). We note that simulation-based objectives, such as that of DSB \citep{de_bortoli_diffusion_2021}, also \emph{train} in a fixed time discretization, inducing a further approximation error.
\end{enumerate}

\paragraph{Iterative fitting.}
Both mean-matching and bridge matching Schr\"odinger bridge methods require an outer IPF step, which samples data to train for each inner neural network optimization step (\autoref{tab:algs_comparison}). At each step, new data is resampled to determine the boundary distributions of the next iteration (\autoref{alg:sf2m-looped-train}). This multiplies the training time by $L$ (the number of outer loops), and cannot guarantee that that the neural networks match the right endpoint marginals until the outer IPF step has converged. 

In contrast, [SF]$^2$M is always optimizing for the correct marginals, but at the cost of a biased Schr\"odinger bridge due to bias in the minibatch OT. We hypothesize that this is one of the reasons that [SF]$^2$M has better generative modeling performance than mean-matching or bridge-matching methods.

In summary, we recommend using DSBM / IDBM when approximating the true Schr\"odinger bridge is important and the computational budget is large. We recommend using [SF]$^2$M when the computational budget is small or when the generative performance is more important than matching the true Schr\"odinger bridge process.

\section{ADDITIONAL RESULTS AND ABLATIONS}
\label{sec:additional_results}

\begin{figure}[htb]
     \centering
     \includegraphics[width=\columnwidth]{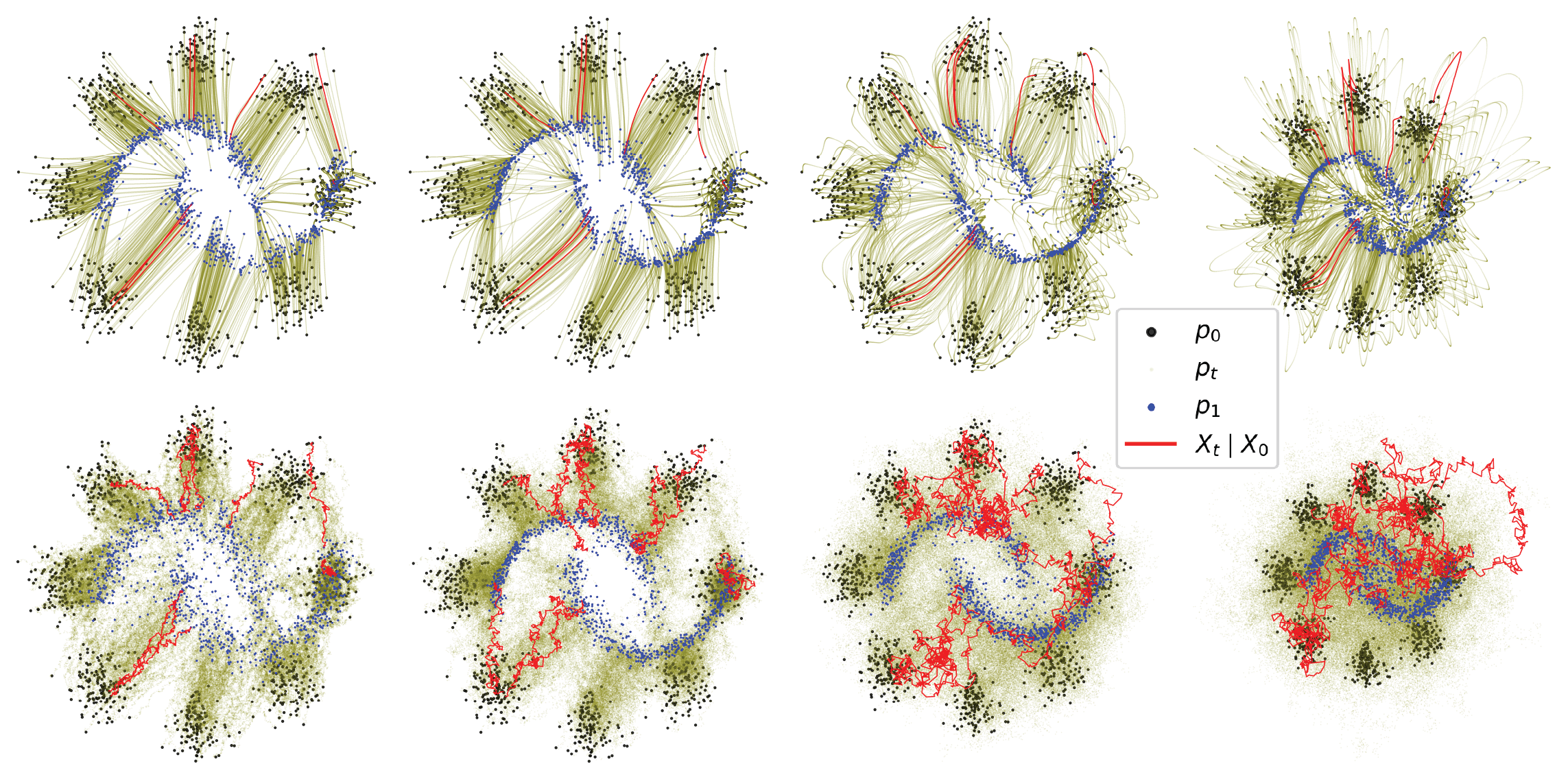}
      \vspace{-0.2cm}
    \caption{Learned ODE (top) and SDE (bottom) simulations for $\sigma \in [0.1, 1, 2, 3]$ from left to right for trained [SF]$^2$M model. The marginals match regardless of the chosen $\sigma$. Trajectory initializations are matched between runs.} 
    \label{fig:sigma-exploration}
\end{figure}
\begin{figure}[t]
     \centering
     \includegraphics[width=\columnwidth]{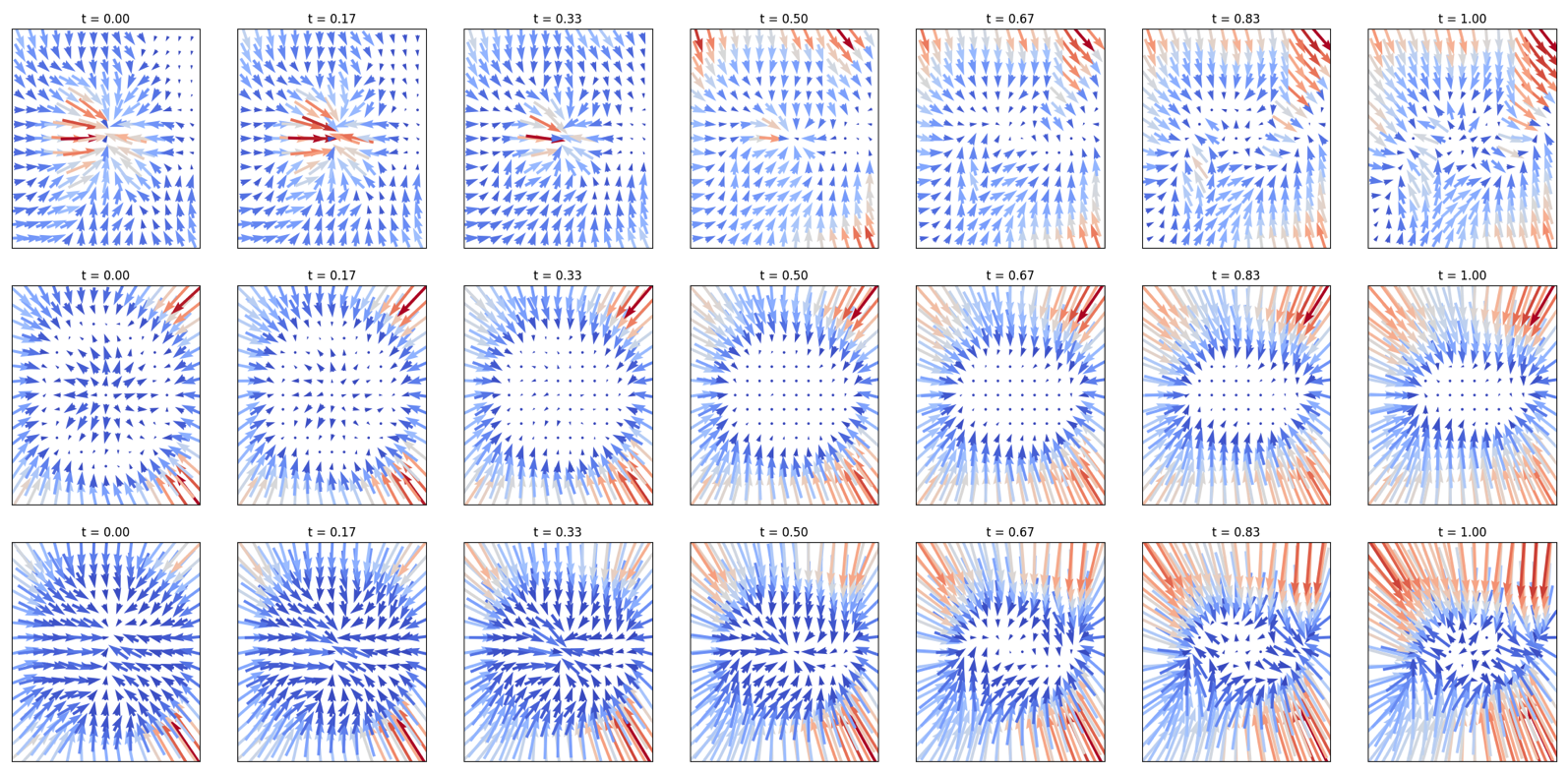}
    \caption{From top to bottom $v_\theta(t, x)$, $s_\theta(t, x)$ and $u_t(x)$ inferred through (\ref{eq:sde_from_sf}) for the 8-Gaussians to moons dataset.}
    \label{fig:8gaussians-moons:vectors}
\end{figure}

\begin{table}[]
    \centering

    \caption{Gaussian-to-Gaussian Schr\"odinger bridges with 10,000 datapoints. Here we text Sinkhorn-Exact which uses the exact OT (default), against the Sinkhorn algorithm for the static OT within batches, and using outer loops where we simulate 10,000 trajectories for further training 20 times following \citet{shi_diffusion_2023}.}
\begin{tabular}{lrlll}
\toprule
     & Dim &                        [SF]$^2$M-Sinkhorn &                  [SF]$^2$M-Exact &           [SF]$^2$M-Exact-Looped \\
\midrule
$\mathrm{KL}(p_1, q_1)$ & 2   &  \textbf{0.002 $\pm$ 0.000} &           0.003 $\pm$ 0.000 &           0.032 $\pm$ 0.009 \\
     & 5   &  \textbf{0.004 $\pm$ 0.001} &           0.007 $\pm$ 0.000 &           0.088 $\pm$ 0.011 \\
     & 20  &  \textbf{0.029 $\pm$ 0.001} &           0.029 $\pm$ 0.002 &           0.293 $\pm$ 0.021 \\
     & 50  &  \textbf{0.122 $\pm$ 0.003} &           0.124 $\pm$ 0.003 &           0.610 $\pm$ 0.033 \\
     & 100 &           0.493 $\pm$ 0.010 &  \textbf{0.486 $\pm$ 0.005} &           1.578 $\pm$ 0.026 \\\midrule
Mean $\mathrm{KL}(p_t, q_t)$ & 2   &  \textbf{0.001 $\pm$ 0.000} &           0.004 $\pm$ 0.000 &           0.006 $\pm$ 0.002 \\
     & 5   &  \textbf{0.004 $\pm$ 0.000} &           0.006 $\pm$ 0.000 &           0.019 $\pm$ 0.003 \\
     & 20  &           0.042 $\pm$ 0.001 &  \textbf{0.028 $\pm$ 0.001} &           0.080 $\pm$ 0.006 \\
     & 50  &           0.276 $\pm$ 0.001 &           0.258 $\pm$ 0.001 &  \textbf{0.243 $\pm$ 0.012} \\
     & 100 &           1.000 $\pm$ 0.007 &           0.977 $\pm$ 0.003 &  \textbf{0.792 $\pm$ 0.008} \\
\bottomrule
\end{tabular}

    \label{supp:tab:gaussian-extended}
\end{table}

\begin{algorithm}
\caption{Looped Minibatch Simulation-Free Score and Flow Matching Training}
\label{alg:sf2m-looped-train}
\begin{algorithmic}
\State {\bfseries Input:} Samplable source and target $q_0, q_1$, number of outer loop iterations $L$, inner loop iterations $I$, cache size $n$, noise term $\sigma$, weighting schedule $\lambda(t)$, initial networks $v_{\theta}$ and $s_\theta$.
\For{Outer loop $l \in [1, \ldots, L]$}
\For{Inner loop $i \in [1, \ldots, I]$}
\If{$l = 0$}
\State $x_0, x_1 \sim q_0^{\otimes m}, q_1^{\otimes m}$
\State $\pi \gets \mathrm{Sinkhorn}(x_0, x_1, 2 \sigma^2)$ \Comment{Or $\mathrm{OT}(x_0, x_1)$ see \autoref{supp:sec:mbot}}
\State $x_0, x_1 \sim \pi^{\otimes m}$ \Comment{Resample OT pairs from $\pi$}
\Else{}
\State $x_0, x_1 \sim \mathbb{T}^{\otimes m}$
\EndIf
\State $t \sim \mathcal{U}(0, 1)$
\State $p_t(x | x_0, x_1) \gets \mathcal{N}(x; t x_1 + (1 - t) x_0, \sigma^2 t (1-t))$
\State $x \sim p_t( x | x_0, x_1)$
\State $\mathcal{L}_{\rm [SF]^2M} \gets \| v_\theta(t, x) - \nothat u_t(x | x_0, x_1)\|^2 + \lambda(t)^2 \|s_\theta(t, x) - \nabla_{x}\log p_t(x|x_0, x_1)||^2$
\State $\theta \gets \mathrm{Update}(\theta, \nabla_\theta \mathcal{L}_{\rm [SF]^2M})$
\EndFor
\State $\mathbb{T}_f \gets (x_0, \hat{x}_1)^{\otimes n // 2}$
\Comment{where $\hat{x}_1$ is sampled according to \autoref{alg:sf2m-inference}}
\State $\mathbb{T}_b \gets (\hat{x}_0, x_1)^{\otimes n // 2}$
\Comment{where $\hat{x}_0$ is sampled according to the backwards analog of \autoref{alg:sf2m-inference}}
\State $\mathbb{T} \gets [\mathbb{T}_f, \mathbb{T}_b]$
\EndFor
\State \Return $v_\theta,s_\theta$
\end{algorithmic}
\end{algorithm}

\subsection[Looped SF2M]{Looped {[SF]}$^2$M}

As previously discussed, the majority of Schr\"odinger bridge algorithms to date have used outer iterations to perform iterative proportional fitting on the continuous distributions. This can create marginals closer to the true Schr\"odinger bridge marginals at the cost of additional computation, and potentially worse generative modelling performance. [SF]$^2$M is the first Schr\"odinger bridge method to approximate Schr\"odinger bridges without performing the iterative proportional fitting in continuous time and space, instead using much more efficient iterations in the static, discrete OT setting. 

However, [SF]$^2$M is compatible with outer looping. In \autoref{supp:tab:gaussian-extended} we see that outer looping (with 20 outer loops following~\citet{shi_diffusion_2023}) produces better Schr\"odinger bridge marginals, but worse marginals at time 1 indicating worse generative performance. We note that outer looping takes much longer to train as [SF]$^2$M effectively takes a single outer loop, and may have advantages even over a single outer loop as shown in \citet{tong_conditional_2023,pooladian_2023_multisample} where the static solution sped up training in the deterministic setting.

Here we accomplish outer looping by simulating 5,000 (stochastic) trajectories from the backwards SDE and 5,000 trajectories from the forwards SDE. We then use the start and end points of these trajectories as samples from the approximate static OT matrix in the next iteration. This algorithm is detailed in \autoref{alg:sf2m-looped-train}. We always set $n$ (the number of samples per outer loop) to the size of the empirical dataset, fix the number of outer loops to 20, and the number inner loops to the [SF]$^2$M without outer loops divided by 20. This gives the same number of gradient descent steps for all methods, but we note that the outer loop methods require simulation for each outer loop. This adds additional computation cost.

\subsection{On the choice of static OT method: Sinkhorn vs.\ Exact}

As the minibatch size gets large, [SF]$^2$M with entropic OT with $\epsilon = 2 \sigma^2$ is the correct choice. However, with minibatching, we want a smaller $\epsilon$. In practice we often use $\epsilon = 0$ corresponding to exact (unregularized) optimal transport. We test this on the Gaussian-to-Gaussian experiment in \autoref{supp:tab:gaussian-extended}. Here we see that [SF]$^2$M with sinkhorn OT works better in low dimensions, but struggles in high dimensions. We believe this is because minibatch-OT effectively adds more entropy in higher dimensions. More experimentation and theory is needed to determine the optimal setting of $\epsilon$ for a given dataset and minibatch size. 

\begin{table}[t!]
    \centering
    \begin{tabular}{@{\hskip3pt}c@{\hskip3pt}@{\hskip3pt}c@{\hskip3pt}@{\hskip3pt}c@{\hskip3pt}@{\hskip3pt}c@{\hskip3pt}@{\hskip3pt}c@{\hskip3pt}@{\hskip3pt}c@{\hskip3pt}@{\hskip3pt}c@{\hskip3pt}}%
\toprule
   \normalsize [SF]$^2$M & I-CFM &  OT-CFM &   DSBM &      FM &     RF \\
\midrule
  4.145 &  4.381 & 4.456 &      4.511 &   4.611 &   6.01 \\
\bottomrule
\end{tabular}
    \caption{CIFAR-10 \textbf{FID} using 100-step Euler integration. First 4 models trained with batch size 128 for 16 A100-hours. Last 2 models from the DSBM paper, trained for $\sim$200 A100-hours. [SF]$^2$M performs better than DSBM (with 1/12 the compute) and the deterministic methods (I-CFM, OT-CFM, FM). Better performance can be obtained by considering higher-order or adaptive step solvers.}
    \label{tab:cifar}
\end{table}
\subsection{Extra experiments on Cifar10}
We have conducted experiments on Cifar 10 using the code from \cite{tong_conditional_2023} and the experimental evaluation from \cite{shi_diffusion_2023}. Results are gathered in Table \ref{tab:cifar}.

\begin{table}[thb]
    \caption{GRN recovery and leave-last-timepoint-out testing using single-cell gene expression over two simulated datasets. We measure performance of predicting the distribution of the final left-out timepoints (2-Wasserstein and radial basis function maximum mean discrepancy) as well accuracy of structure recovery (AUC-ROC and AP).}
    \label{supp:tab:boolode-grn}
    \centering
    \resizebox{\linewidth}{!}{
\begin{tabular}{@{}lcccccccc}
\toprule
& \multicolumn{4}{c}{Bifurcating System} & \multicolumn{4}{c}{Trifurcating System} \\
\cmidrule(lr){2-5}\cmidrule(lr){6-9}
     & $\mathcal{W}_2$ ($\downarrow$) & RBF-MMD ($\downarrow$) & AUC-ROC ($\uparrow$) & AP ($\uparrow$) & $\mathcal{W}_2$ ($\downarrow$) & RBF-MMD ($\downarrow$) & AUC-ROC ($\uparrow$) & AP ($\uparrow$)\\
\midrule
    OT-CFM & \textbf{0.782} $\pm$ \textbf{0.105} & \textbf{0.054} $\pm$ \textbf{0.004} & --- & --- & 0.921 $\pm$ 0.142 & 0.068 $\pm$ 0.006 & --- & --- \\
    {[SF]}$^2$M & 0.791 $\pm$ 0.097 & 0.056 $\pm$ 0.005 & --- & --- & 0.932 $\pm$ 0.159 & \textbf{0.062} $\pm$ \textbf{0.006} & --- & --- \\
    NGM-{[SF]}$^2$M$_{\sigma=0}$ & 0.783 $\pm$ 0.112 & 0.055 $\pm$ 0.006 & \textbf{0.786} $\pm$ \textbf{0.081} & \textbf{0.521} $\pm$ \textbf{0.160} & \textbf{0.912} $\pm$ \textbf{0.161} & 0.064 $\pm$ 0.005 & \textbf{0.764} $\pm$ \textbf{0.066} & \textbf{0.485} $\pm$ \textbf{0.105} \\
    NGM-{[SF]}$^2$M$_{\sigma=0.1}$ & 0.835 $\pm$ 0.089 & 0.064 $\pm$ 0.011 & 0.723 $\pm$ 0.014 & 0.444 $\pm$ 0.030 & 1.049 $\pm$ 0.195 & 0.080 $\pm$ 0.014 & 0.731 $\pm$ 0.077 & 0.453 $\pm$ 0.091 \\
    NGM-{[SF]}$^2$M$_{\sigma=0.01}$ & 0.813 $\pm$ 0.101 & 0.064 $\pm$ 0.008 & 0.715 $\pm$ 0.047 & 0.442 $\pm$ 0.033 & 0.956 $\pm$ 0.121 & 0.069 $\pm$ 0.005 & 0.726 $\pm$ 0.081 & 0.451 $\pm$ 0.094 \\
    NGM-{[SF]}$^2$M$_{\sigma=0.001}$ & 0.844 $\pm$ 0.063 & 0.082 $\pm$ 0.018 & 0.699 $\pm$ 0.043 & 0.418 $\pm$ 0.060 & 1.005 $\pm$ 0.150 & 0.094 $\pm$ 0.014 & 0.725 $\pm$ 0.080 & 0.445 $\pm$ 0.082 \\
\bottomrule
    Spearman & --- & --- & 0.755 $\pm$ 0.003 & 0.438 $\pm$ 0.002 & --- & --- & 0.718 $\pm$ 0.005 & 0.413 $\pm$ 0.005 \\
    Pearson & --- & --- & 0.744 $\pm$ 0.000 & 0.415 $\pm$ 0.000 & --- & --- & 0.710 $\pm$ 0.002 & 0.405 $\pm$ 0.002 \\
    DREMI~\citep{krishnaswamy_conditional_2014} & --- & --- & 0.594 $\pm$ 0.017 & 0.293 $\pm$ 0.011 & --- & --- & 0.419 $\pm$ 0.021 & 0.205 $\pm$ 0.007 \\
    Granger~\citep{granger_investigating_1969} & --- & --- & 0.664 $\pm$ 0.013 & 0.421 $\pm$ 0.043 & --- & --- & 0.613 $\pm$ 0.048 & 0.343 $\pm$ 0.039 \\
\bottomrule
\end{tabular}
}
\end{table}

\section{SCHR\"ODINGER BRIDGES WITH VARYING DIFFUSION RATE}\label{sec:nonconstant_sigma}
While we consider constant diffusion for the majority of this paper, the theory also applies for time varying diffusion with the variation as specified in the following Lemma.
\begin{lemma}[Brownian bridge with time-varying diffusion]
Suppose $\rvx_t$ is a stochastic process with values in functions $[0,1]\to\R^d$, defined by initial conditions $\rvx_0=\rva$ and SDE $d\rvx_t=\sigma(t)\,d\rvw_t$, where $\sigma(t)$ is continuous and positive on $(0,1)$. Define $F(t)=\int_0^t\sigma^2(s)\,ds$. Then $\rvx_t$ satisfies
\begin{align*}
\rvx_t&\sim\gN(\rva,F(t))\\
\rvx_t\mid\rvx_1=\rvb&\sim\gN\left(\rva+(\rvb-\rva)\frac{F(t)}{F(1)},F(t)-\frac{F(t)^2}{F(1)}\right).
\end{align*}
\end{lemma}
\begin{proof}
The constraints on $\sigma$ guarantee that $F$ has a unique inverse $F^{-1}$ on $[0,F(1)]$. Consider the process 
\[\rvy_t=\frac{\rvx_{F^{-1}(F(1)t)}-\rva}{\sqrt{F(1)}},\] which is equivalently characterized by \[\rvx_t=\rva+\sqrt{F(1)}\rvy_{F(t)/F(1)}.\]
A straightforward computation using the chain rule shows that $\rvy_t$ satisfies $\rvy_0=\boldsymbol{0}$ and $d\rvy_t=d\rvw_t$, \ie, $\rvy_t$ is Brownian motion. 

By standard facts about Brownian bridges, we have
\begin{align*}
\rvy_t&\sim\gN(0,t)\\
\rvy_t\mid\rvy_1=\frac{\rvb-\rva}{\sqrt{F(1)}}&\sim\gN\left(\frac{\rvb-\rva}{\sqrt{F(1)}} t,t(1-t)\right).
\end{align*}
The result follows by applying the reverse transformation to obtain the marginals of $\rvx_t$.
\end{proof}

Because Markovization commutes with time reparametrization, and the SB is the Markovization of a mixture of Brownian bridges, this immediately implies:
\begin{corollary}
The solution $p$ to the SB problem with reference process $d\rvx_t=\sigma(t)\,d\rvw_t$ has marginal $p(x_0,x_1)=\pi_{2F(1)}(x_0,x_1)$.
\end{corollary}

\section{EXPERIMENTAL DETAILS}\label{supp:sec:exp_details}

\begin{algorithm}
\caption{Minibatch Optimal Transport Simulation-Free Score and Flow Matching Training}
\label{alg:sf2m-detailed-train}
\begin{algorithmic}
\State {\bfseries Input:} Samplable source and target $q_0, q_1$, noise term $\sigma$, weighting schedule $\lambda(t)$, initial networks $v_{\theta}$ and $s_\theta$.
\While{Training}

\State $t \sim \mathcal{U}(0, 1)$
\State $x_0, x_1 \sim q_0^{\otimes m}, q_1^{\otimes m}$
\State $\pi \gets \mathrm{Sinkhorn}(x_0, x_1, 2 \sigma^2)$ \Comment{Or $\mathrm{OT}(x_0, x_1)$ see \autoref{supp:sec:mbot}}
\State $x_0, x_1 \sim \pi^{\otimes m}$ \Comment{Resample OT pairs from $\pi$}
\State $p_t(x | x_0, x_1) \gets \mathcal{N}(x; t x_1 + (1 - t) x_0, \sigma^2 t (1-t))$
\State $x \sim p_t( x | x_0, x_1)$
\State $\mathcal{L}_{\rm [SF]^2M} \gets \| v_\theta(t, x) - u_t(x | x_0, x_1)\|^2 + \lambda(t)^2 \|s_\theta(t, x) - \nabla_{x}\log p_t(x|x_0, x_1)||^2$
\Comment {see (\ref{eq:brownian_bridge_flow}, \ref{eq:cond_sf2m})}
\State $\theta \gets \mathrm{Update}(\theta, \nabla_\theta \mathcal{L}_{\rm [SF]^2M})$
\EndWhile
\State \Return $v_\theta,s_\theta$
\end{algorithmic}
\end{algorithm}

\subsection{Implementation details and settings}

Throughout we use networks of three layers of width 64 with SeLU activations~\citep{klambauer_self-normalizing_2017} except for the 1000 dimensional experiment where we use width 256, and for the neural graphical model (NGM) model used in the gene regulatory network recovery task. For optimization we use ADAM-W~\citep{loshchilov_decoupled_2019} with learning rate $10^{-3}$ and weight decay $10^{-5}$. We train for 1,000 epochs unless otherwise noted. The flow network and score networks always have exactly the same structure. We provide a more detailed picture of the algorithm used in \autoref{alg:sf2m-detailed-train}. For sampling we always take 100 integration steps with either the Euler integrator for ODEs or Euler-Maruyama integrator for SDEs, except for the Gaussian experiments where we take 20 steps to match the setup of \citet{de_bortoli_diffusion_2021,shi_diffusion_2023}. When there are multiple timepoints (\eg, single-cell trajectories) we take 100 steps between each timepoint for a total of $100 k$ steps for all methods. We use $\sigma = 1$ unless otherwise noted.

\subsubsection{Weighting schedule $\lambda(t)$}
Similar to other score and diffusion-based losses, the [SF]$^2$M loss is defined with a weighting schedule $\lambda(t)$. Since $\nabla \log p_t(x | z)$ goes to infinity as $t$ tends to zero or one, we must standardize the loss to be roughly even over time. We set $\lambda(t)$ such that the target has zero mean and unit variance, \ie,\ we predict the noise added in sampling $x$ from $\mu_t$ before multiplying by $\sigma \sqrt{t (1-t)}$. This leads to the setting:
\begin{equation}
    \lambda(t) = \sigma_t = \sigma \sqrt{t ( 1- t )}
\end{equation}

this weighting schedule ensures that the regression target for $s_\theta$ is distributed $\mathcal{N}(0,1)$.

We also experiment with directly regressing $s_\theta$ against the scaled target $\frac{1}{2}\sigma^2 \nabla_x \log p_t(x | z)$, with a different weighting function
\begin{equation}\label{eq:lambda_schedule}
    \lambda(t) = \frac{2}{\sigma^2} \sigma_t = \frac{2 \sqrt{t (1- t)}}{\sigma}
\end{equation}
which also ensures the regression target for $\sigma$ is distributed $\mathcal{N}(0,1)$. With this $\lambda$, and since $\nabla_x \log p_t(x | z)$ with a Gaussian $p_t(x | z)$ can be simplified to $-\epsilon / \sigma_t$ where $\epsilon \sim \mathcal{N}(0,1)$ we have the simplified objective:
\begin{align}
    \lambda(t)^2 \left \| s_\theta(t, x) - \nabla_x \log p_t(x | x_0, x_1)\right\|^2 
    &= \| \lambda(t) s_\theta(t, x) - \lambda(t) \nabla_x \log p_t(x | x_0, x_1) \|^2 \\
    &= \left \| \lambda(t) s_\theta(t, x) + \epsilon_t \right \|^2 \label{eq:simplified_objective}
\end{align}

This is also numerically stable as it avoids dividing by anything that approaches zero. We leave improved weighting schedules to future work. In practice we use the schedule in (\ref{eq:lambda_schedule}) and the simplified objective in (\ref{eq:simplified_objective}).

\subsubsection{Static optimal transport}
As discussed in \autoref{supp:sec:mbot} it is sometimes preferable to use exact optimal transport instead of entropic optimal transport. In addition to the ``entropy'' added by minibatching, there are also numerical and practical considerations. In practice, we find that for batches of size $<10000$, implementations of exact OT through the Python Optimal Transport package (POT)~\citep{flamary_pot_2021} are often faster than implementations of the Sinkhorn algorithm~\citep{cuturi_sinkhorn_2013}, as the Sinkhorn algorithm is known to have numerical difficulties and needs many iterations for good approximation of the true entropic transport for small values of $\sigma$. %

\subsection{Computational resources}\label{supp:sec:resources}

All experiments were performed on a shared heterogeneous high-performance-computing cluster. This cluster is primarily composed of GPU nodes with RTX8000, A100, and V100 Nvidia GPUs, and CPU nodes with 32 and 64 CPUs.

\subsection{Two-dimensional experimental details}

For the two-dimensional experiments we follow the setup from \citet{shi_diffusion_2023}. This is adapted from the setup of \citet{tong_conditional_2023} except using larger test sets for lower variance in the empirical estimation of the Wasserstein distance. We use a training set size of 10,000, a validation set size of 10,000, and a test set size of 10,000 for all methods and models. 

We evaluate the empirical 2-Wasserstein distance for 10,000 forward samples from our model pushing the source distribution to the target. The number reported for $\mathcal{W}_2$ is then 
\begin{equation}
    \mathcal{W}_2 = \left ( \min_{\pi \in U(\hat{p}_1, q_1)} \int \|x - y\|_2^2 d \pi(x, y) \right )^{1/2}
\end{equation}
where $\hat{p}_1$ is sampled via \autoref{alg:sf2m-inference}, and $q_1$ is the test set.

As mentioned in the main text, we also measure the Normalized Path Energy. We note that this is only defined for ODE integration, hence for stochastic methods, (\ie, [SF]$^2$M, DSBM) we measure the normalized path energy of the probability flow ODE (see (\ref{eq:pfode})). The normalized path energy measures the relative deviation of the path energy of the model ($\int \|v_\theta(t, x)\|^2$) to the path energy of the optimal paths ($\mathcal{W}_2^2(q_0, q_1)$), which is equivalent to the squared 2-Wasserstein distance between the test set source and the test set target. More formally, the normalized path energy can be calculated as
\begin{equation}
    NPE(q_0, q_1, v_\theta) = \frac{| \E_{x(0) \sim q_0} \int \|v_\theta(t, x_t)\|^2 dt - \mathcal{W}_2^2(q_0, q_1) |}{\mathcal{W}_2^2}
\end{equation}
where $x_t$ is the solution of the probability flow ODE with $dx = v_\theta(t, x) dt$ with initial condition $x_0$. This measures how close the paths defined by $v_\theta(t, x)$ are to the optimal transport paths in terms of average energy. We note that measuring the path energy rather than the length ($\mathcal{W}_2^2$ instead of $\mathcal{W}_1$) has the additional benefit that the energy differentiates between models that follow the same paths, but at different rates, encouraging constant rate models.

\subsection{Gaussian-to-Gaussian experimental details}

Similar to experiments in \citet{de_bortoli_diffusion_2021,shi_diffusion_2023}, we train on a sample of 10,000 points from a $\mathcal{N}(-0.1, I)$ source to 10,000 points sampled from a $\mathcal{N}(0.1, I)$ target, for dimensions 5, 20, and 50. We could not find all the details of those previous experiments but we match what we can here. Since we know the closed-form solution to the Schr\"odinger bridge from \citet{mallasto_entropy_2021,bunne_recovering_2022} $q_t$, we can compare our estimate of the marginal at time $t$ ($\hat{p}_t$) with the true distribution at time $t$. In particular, we compare a Gaussian approximation of $\hat{p}_t$, $\tilde{p}_t$ with mean and covariance estimated from 10,000 samples of the model with $q_t$ which has the form 
\begin{equation}
    q_t(x) = \mathcal{N}\left (x; 0.2 t - 0.1, ( t ( 1-t) \sqrt{4 + \sigma^4} + (1-t)^2 + t^2) I \right )
\end{equation}
with the $\mathrm{KL}$ divergence.

We compare either the average $\mathrm{KL}$ over 21 timepoints (including the start and end timepoints) \ie,
\begin{equation}
    \overline{\mathrm{KL}} = \frac{1}{20} \sum_{k=0}^{20} \mathrm{KL}(\tilde{p}_{k / 20} || q_{k / 20})
\end{equation}
to measure how closely the learned flow matches the true Schr\"odinger bridge marginals, and we also compare the $\mathrm{KL}$ divergence at time $t=1$, to measure the performance as a generative model. 
\begin{equation}
    \mathrm{KL}(\tilde{p}_1 || q_1)
\end{equation}

\subsection{Waddington's landscape experimental details}
We use two Waddington landscapes, one Gaussian to two Gaussians and cross-sectional measurements of Embryoid Body (EB) data, to demonstrate the versatility of [SF]$^2$M for trajectory inference. The three dimensions of the landscape are space, time and potential.

More specifically, the space dimension is evolved according to the drift of the SDE by $u_t(x) = -\nabla_x W(t,x)$. The potential dimension is the Waddington's landscape by $W := E_v + \frac{1}{2} g(t)^2 E_s$, where $v_\theta(t, x) = -\nabla_x E_v(t, x)$ and $s_\theta(t, x) = - \nabla_x E_s(t, x)$ are the flow and score for Langevin dynamics. $E_s$ and $E_v$ are both parameterized by three layer neural networks of width 64, the only difference with our standard implementation of $v_\theta$ and $s_\theta$ is that they have output dimension width of one instead of $d$. 

For the experiment of one Gaussian to two Gaussians, we train on a sample of 256 points from the one-dimensional source $\mathcal{N}(0, 0.1)$ to 256 points sampled from the one-dimensional target $\mathcal{N}(-1, 0.1)\cup \mathcal{N}(1,0.1)$ for 10,000 steps. We then plot 20 trajectories from the source to the target with the potentials following from the gradient descend of $W$.

For the cross-sectional measurements from the embryoid body (EB) data, we first embed the data in one dimension with the non-linear dimensionality reduction technique PHATE~\citep{moon_visualizing_2019}, which we then whiten to ensure the data is at a reasonable scale for the neural network initialization. We train the [SF]$^2$M model following \autoref{alg:sf2m-trajectory-train} for 50,000 steps and plot 100 stochastic trajectories along with the height of $W(t, x)$ normalized so to gradually descend over time.

\begin{algorithm}
\caption{Trajectory Simulation-Free Score and Flow Matching Training}
\label{alg:sf2m-trajectory-train}
\begin{algorithmic}
\State {\bfseries Input:} Samplable source and target $Q=\{q_0, \cdots, q_{K-1}\}$, noise term $\sigma$, weighting schedule $\lambda(t)$, initial networks $v_{\theta}$ and $s_\theta$.
\While{Training}

\For {$k \in [0, \cdots, K-2]$}
    \State $x_k, x_{k+1} \sim q_k^{\otimes m}, q_{k+1}^{\otimes m}$
    \State $\pi \gets \mathrm{Sinkhorn}(x_k, x_{k+1}, 2 \sigma^2)$ \Comment{Or $\mathrm{OT}(x_k, x_{k+1})$ see \autoref{supp:sec:mbot}}
    \State $\mathbb{X}_k \sim \pi^{\otimes m}$ \Comment{Resample OT pairs from $\pi$}
\EndFor
\State $k \sim \mathcal{U}(1, K)^{\otimes m}$
\State $t \sim \mathcal{U}(0, 1)^{\otimes m}$
\State $x_0^i, x_1^i \gets \mathbb{X}_k^i$
\State $p_t(x | x_0, x_1) \gets \mathcal{N}(x; t x_1 + (1 - t) x_0, \sigma^2 t (1-t))$
\State $x \sim p_t( x | x_0, x_1)$
\State $\mathcal{L}_{\rm [SF]^2M} \gets \| v_\theta(t+k, x) - u_t(x | x_0, x_1)\|^2 + \lambda(t) \|s_\theta(t+k, x) - \nabla_{x}\log p_t(x|x_0, x_1)||^2$
\State $\theta \gets \mathrm{Update}(\theta, \nabla_\theta \mathcal{L}_{\rm [SF]^2M})$
\EndWhile
\State \Return $v_\theta,s_\theta$
\end{algorithmic}
\end{algorithm}

\subsection{Single-cell interpolation experimental details}\label{supp:sec:exp_details:landscape}

Here we perform two comparisons, the first matching the setup of \citet{tong_trajectorynet_2020} in low dimensions and the second exploring higher dimensional single-cell interpolation. Following \citet{huguet_geodesic_2022}, we repurpose the CITE-seq and Multiome datasets from a recent NeurIPS competition for this task~\citep{burkhardt_multimodal_2022} as well as the Embryoid-body data from \citet{moon_visualizing_2019,tong_trajectorynet_2020}, which has 5 population measurements over 30 days. 

For the Embryoid body (EB) data, we use the same processed artifact which contains the first 100 principal components of the data. For our tests in \autoref{tab:sc-dynamics-fived}, we truncate to the first five dimensions, then whiten each dimension following~\citet{tong_trajectorynet_2020} before interpolation. For the Embryoid body (EB) dataset which consists of 5 timepoints collected over 30 days we train separate models leaving out times $1,2,3$ in turn. During testing we push forward all observed points $X_{t-1}$ to time $t$ then measure the 1-Wasserstein distance between the predicted and true distribution. 

For the Cite and Multi datasets these are sourced from the Multimodal Single-cell Integration challenge at NeurIPS 2022, a NeurIPS challenge hosted on Kaggle where the task was multi-modal prediction~\citep{burkhardt_multimodal_2022}. Here, we repurpose this data for the task of time series interpolation. Both of these datasets consist of four timepoints from CD34+ hematopoietic stem and progenitor cells (HSPCs) collected on days 2, 3, 4, and 7. For more information and the raw data see the competition site.\footnote{\url{https://www.kaggle.com/competitions/open-problems-multimodal/data}} We preprocess this data slightly to remove patient specific effects by focusing on a single donor (donor 13176).

Since these data have the full (pre-processed) gene level single-cell data, we try interpolating on higher-dimensional unwhitened principle components, and on the first 1000 highly variable genes, which is a standard preprocessing step in single-cell data analysis. To our knowledge, [SF]$^2$M is the first method to scale to the gene space of single-cell data. In \autoref{tab:sc-dynamics}, we again measure the 1-Wasserstein distance between the push forward predicted distribution and the ground truth distribution.

\subsubsection{Geodesic ground costs}

We also introduce the Geodesic Sinkhorn method from \citet{huguet_geodesic_2022} for dynamic Schr\"odinger bridge interpolation. Here the cost is a geodesic cost based on a $k$-nearest-neighbour graph between cells.

\subsection{Gene regulatory network recovery experimental details}\label{supp:sec:exp_details:grn}

\begin{figure}[t]
     \centering
     \includegraphics[width=0.75\columnwidth]{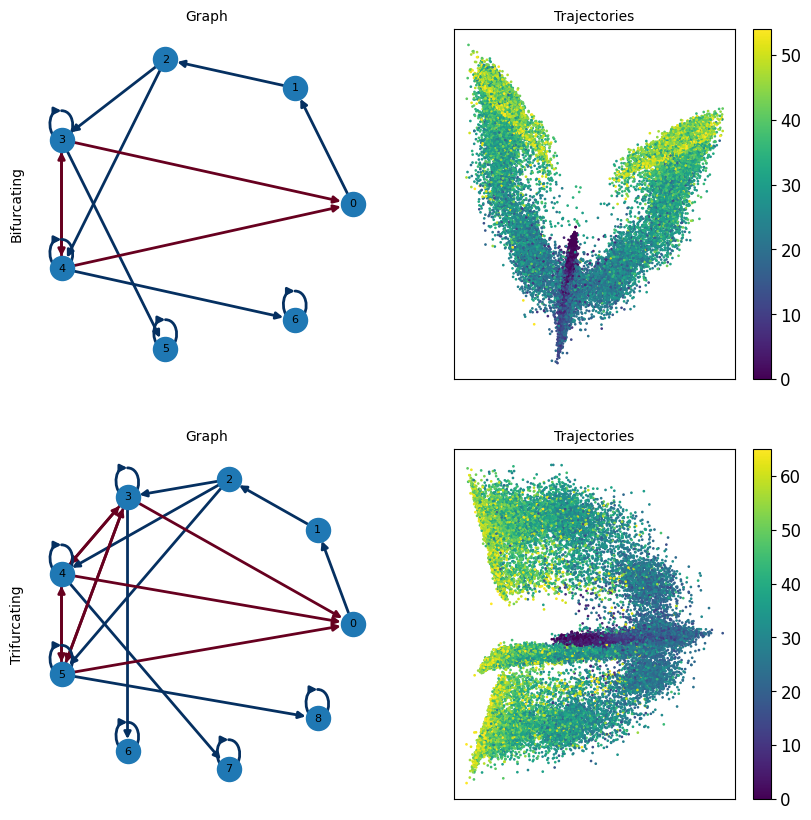}
    \caption{Simulated single-cell trajectories given synthetic GRNs emulating bifurcating (top) and trifurcating (bottom) systems. GRNs contain directed edges of Boolean relationships between genes. For example, a red edge between gene 3 and gene 0 (top left) indicates that the rule for gene 0 is (\textit{not} gene 3). Likewise, a blue edge between gene 6 and gene 4 indicates that rule for gene 6 is (gene 4 and gene 6). This follows from the procedure for defining synthetic GRNs using the BoolODE framework \citep{pratapa2020benchmarking}. The color bar indicates the scale of temporal progression.}
    \label{fig:boolode}
\end{figure}

Using the neural graphical model (NGM), we can parameterize the gene-gene interaction graph directly within the ODE drift model $v_\theta(t, x)$. To do so, following from \citet{bellot_neural_2022} we can define:
\begin{equation}
    v_{\theta_j}(t, x) = \phi(\cdots \phi(\phi(x \theta^{(1)}_j)\theta^{(2)}_j) \cdots)\theta^{(K)}_j, 
    \qquad j=1, \dots, d,
    \label{eq:ngm}
\end{equation}
where $\theta^{(1)}_j \in \mathbb{R}^{d \times h}$ represents a continuous adjacency matrix of the gene-gene interactions, $\theta^{(k)}_j \in \mathbb{R}^{h \times h}, k = 2, \dots, K-1$ are parameters of each proceeding hidden layer, $\theta^{(K)}_j \in \mathbb{R}^{h \times 1}$, $x \in \mathbb{R}^d$, and $\phi(\cdot)$ is an activation function. Then we can consider $v_\theta(t, x) = (v_{\theta_1}(t, x), \dots, v_{\theta_d}(t, x))$ as $h$ ensembles over structure $\theta^{(1)}$. We can then use Algorithm~\autoref{alg:sf2m-train} to train the NGM model with the addition of an $L^1$ penalty over structure to enforce sparsity on gene-gene interactions, i.e $\lambda_1\| \theta^{(1)} \|_1$. We include bias terms in our implementation of (\ref{eq:ngm}).

Using BoolODE \citep{pratapa2020benchmarking}, we generate simulated single-cell gene expression trajectories for a bifurcating system and a trifurcating system. For the bifurcating system we consider $7$ synthetic genes and generate trajectories over $1000$ cells using a simulation time of $5$ and an initial condition on gene $1$ at a value of $1$. We post-process the data and sub-sample to $55$ timepoints and scramble the cell pairing to emulate real-world data. For the trifurcating system we consider $9$ synthetic genes and generate trajectories over $800$ cells using a simulation time of $6$ and an initial condition on gene $1$ at a value of $1$. We post-process the data and sub-sample to $66$ timepoints and scramble the cell pairing to emulate real-world data. We use a train-test data split of $\{ 0.8, 0.2 \}$ respectively, and leave out the end timepoints for trajectory prediction evaluation. We the show underlying synthetic GRNs and simulated single-cell trajectories in \autoref{fig:boolode}.

For OT-CFM (\ie, {[SF]}$^2$M with $\sigma = 0$), we parameterize the NGM model with two hidden layers where $\theta^{(1)}_j \in \mathbb{R}^{d \times h}$ with $h=100$ and $d$ represents the number of input genes. Then the second layer (\ie, $k=2$) is $\theta^{(2)}_j \in \mathbb{R}^{h \times 1}$. We use this parameterization for both the bifurcating system and trifurcating systems. For {[SF]}$^2$M with $\sigma > 0$, we use two heads stemming from $\theta^{(1)}_j$ for the flow matching model and score matching model, respectively. Specifically, we use an additional layer $\tilde{\theta}^{(2)}_j \in \mathbb{R}^{h \times 1}$ such that $s_{\theta_j}(x, t) = \phi(\phi(x \theta^{(1)}_j)\theta^{(2)}_j), j=1, \dots, d$. We use the SeLU activation functions for both models. To train {[SF]}$^2$M and NGM-{[SF]}$^2$M models on the bifurcating system, we use the Adam optimizer with a learning rate of $0.01$ and batch size of $128$ and use $\lambda_1 = 10^{-5}$. On the trifurcating system, we use the Adam optimizer with a learning rate of $0.01$ and batch size of $64$ and use $\lambda_1 = 10^{-6}$. We generate results for $5$ model seeds. For baseline methods (\ie, Spearman, Pearson, DREMI, and Granger) we generate results over $5$ cell-pair scramble seeds. To evaluate GRN recovery performance, we compute the area under the receiver operator characteristic (AUC-ROC) and average precision (AP) scores of the predicted GRNs compared to the ground truth GRNs used for generating the simulated data. We mask out the diagonal elements (self regulation loops) of the predicted and ground truth GRNs for computing the AUC-ROC and AP. We provide the full results of the GRN recovery experiments in \autoref{supp:tab:boolode-grn}.
\end{document}